\documentclass[letterpaper]{article} 
\usepackage{aaai2026}  
\usepackage{times}  
\usepackage{helvet}  
\usepackage{courier}  
\usepackage[hyphens]{url}  
\usepackage{graphicx} 
\usepackage{natbib}  
\usepackage{caption} 
\usepackage{algorithm}
\usepackage{algorithmic}

\usepackage{amsmath}
\usepackage{amssymb}

\usepackage{booktabs}  
\usepackage{amsthm}
\newtheorem{lemma}{Lemma} 
\newtheorem{theorem}{Theorem}
\newtheorem{corollary}{Corollary}
\newtheorem{assumption}{Assumption}

\usepackage[table]{xcolor}  
\usepackage{tabularx} 

\DeclareMathOperator*{\argmax}{argmax}
\usepackage{newfloat}
\usepackage{listings}
\DeclareCaptionStyle{ruled}{labelfont=normalfont,labelsep=colon,strut=off} 
\floatstyle{ruled}
\newfloat{listing}{tb}{lst}{}
\floatname{listing}{Listing}
\title{CEC‑Zero: Zero‑Supervision Character Error Correction with Self‑Generated Rewards}

\author {
    Zhiming Lin\textsuperscript{\rm 1}\equalcontrib,
    Kai Zhao\textsuperscript{\rm 2}\equalcontrib,
    Sophie Zhang\textsuperscript{\rm 3},
    Peilai Yu\textsuperscript{\rm 4},
    Canran Xiao\textsuperscript{\rm 5}\thanks{Corresponding author.}
}
\affiliations {
    \textsuperscript{\rm 1}School of Business, Nankai University, Tianjin, China\\
    \textsuperscript{\rm 2}Hawkesbury Institute for the Environment, Western Sydney University, Sydney, Australia\\
    \textsuperscript{\rm 3}Shanghai High School International Division, Shanghai, China\\
    \textsuperscript{\rm 4}Ludwig Maximilian University of Munich, Munich, Germany\\
    \textsuperscript{\rm 5}School of Cyber Science and Technology, Shenzhen Campus of Sun Yat-sen University, Shenzhen, China\\
    nklinzhiming@gmail.com, w784204411@gmail.com, blinkzen912@gmail.com, peilai.yu@campus.lmu.de, xiaocanran999@gmail.com
}

\usepackage{bibentry}

\begin{document}

\maketitle

\begin{abstract}
Large-scale Chinese spelling correction (CSC) remains critical for real-world text processing, yet existing LLMs and supervised methods lack robustness to novel errors and rely on costly annotations. We introduce CEC-Zero, a zero-supervision reinforcement learning framework that addresses this by enabling LLMs to correct their own mistakes. CEC-Zero synthesizes errorful inputs from clean text, computes cluster-consensus rewards via semantic similarity and candidate agreement, and optimizes the policy with PPO. It outperforms supervised baselines by 10--13 F$_1$ points and strong LLM fine-tunes by 5--8 points across 9 benchmarks, with theoretical guarantees of unbiased rewards and convergence. CEC-Zero establishes a label-free paradigm for robust, scalable CSC, unlocking LLM potential in noisy text pipelines.
\end{abstract}


\section{Introduction}\label{sec:intro}

Large--scale Chinese spelling correction (CSC) has resurfaced as a
critical bottleneck for real--world text processing pipelines in search,
customer--service, health services and educational applications~\citep{diao-etal-2025-temporal, yao2023ndc, wang2025medical, jiang2025transforming,xiao2025curiosity}.
While recent large language models (LLMs) exhibit
impressive general linguistic competence, their sentence--level accuracy on
open--domain CSC benchmarks still lags behind practical requirements,
especially under domain shift~\citep{zhang2025enhancing,tong2025does}.
Closing this gap is essential for unleashing the full potential of
LLM--powered natural--language interfaces in the Chinese marketplace~\citep{diao-etal-2024-learning,Diao_2025_WACV,xiao2025diffusion}.

Unfortunately, increasing model scale alone does not solve CSC.
The task is uniquely demanding:
(i)~\emph{character complexity}---errors arise from homophones,
near--glyph characters, and character splitting;
(ii)~\emph{label scarcity}---collecting balanced, up--to--date
annotations is prohibitively costly because valid corrections are often
non--unique.
Consequently, standard supervised fine--tuning (SFT) or prompt engineering
delivers brittle performance and incurs continual re‐training overhead.

\begin{figure}[ht]
    \centering
    \includegraphics[width=0.48\textwidth]{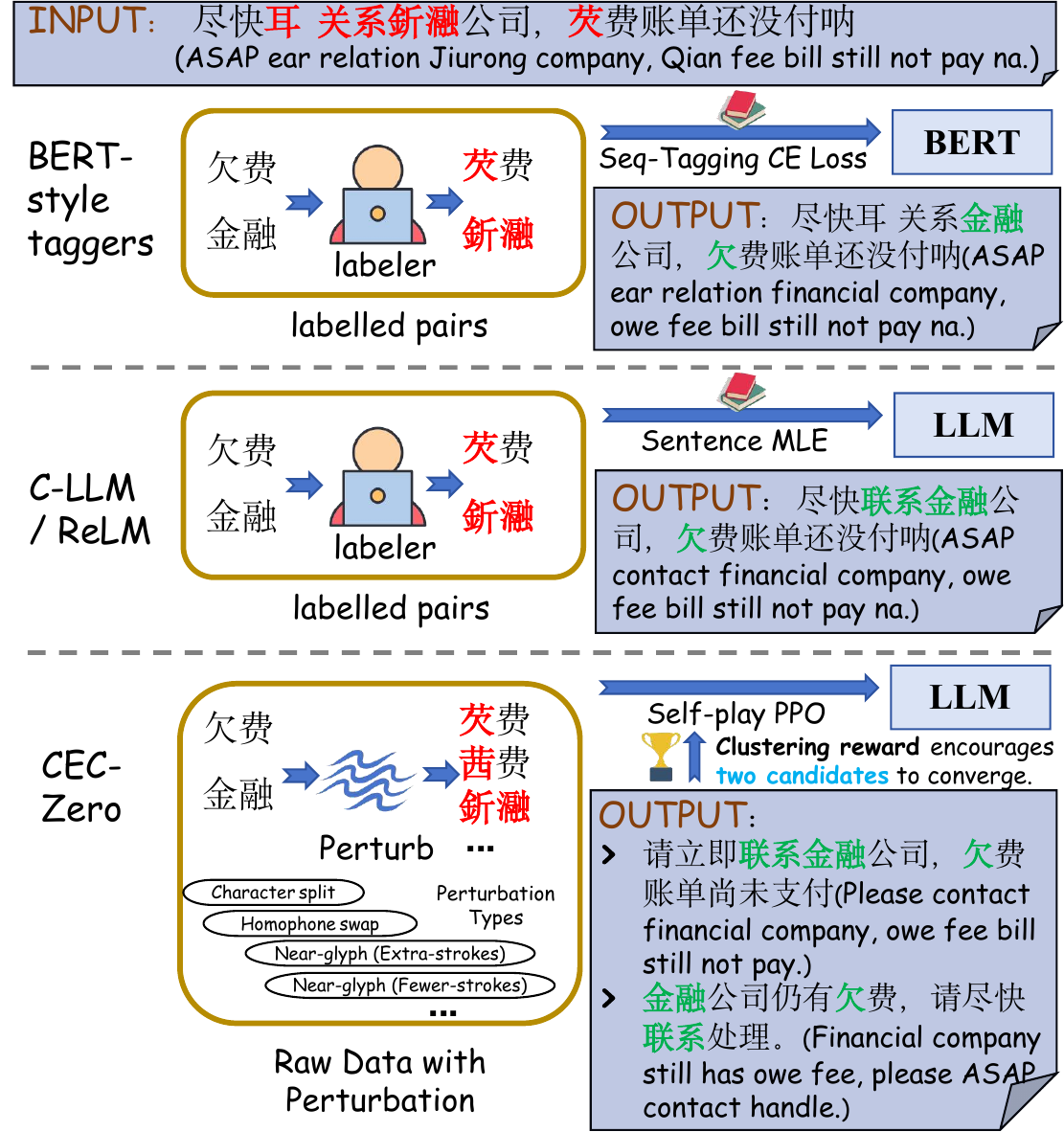}
    \caption{Three routes to Chinese spelling correction.
    BERT taggers rely on token‑level labels and can only perform one‑to‑one glyph swaps, while existing LLM-based methods train on the same pairs with sentence‑level MLE yet still learns by teacher forcing.  \textbf{CEC‑Zero} instead self‑perturbs raw sentences and optimises with PPO, yielding robust label‑free correction.}
    \label{fig:teaser}
\end{figure}

Early solutions framed CSC as sequence labeling on BERT--style encoders
\cite{hong2019faspell,ji2021spellbert}, but those models
implicitly memorise a narrow set of error patterns.
Subsequent work introduced
soft--masking \cite{zhang2020spelling}, multi--task learning with phonetic clues
\cite{li2022improving}, and character‑level LLM(C-LLM) fine‑tuning
\cite{li2024c}; yet they still rely on static,
human‐annotated corpora.
Recent advances take one step toward self–supervision, but still leave
critical gaps(see Figure~\ref{fig:teaser}).
\textit{Rephrasing Language Modeling} (ReLM)~\cite{liu2024chinese} reframes
CSC as sentence–level re‑phrasing, alleviating the token‑to‑token
over‑conditioning of earlier taggers; nevertheless it is still
trained on paired error–correction sentences and supplies no generic
reward for unseen error patterns~\citep{huang2024gaussianmarker,li2024variational}.
Conversely, \textit{Test‑Time Reinforcement Learning} (TTRL)%
~\cite{zuo2025ttrl} derives label‑free rewards from majority voting, but
its formulation assumes deterministic reasoning tasks (e.g.\ maths, code)
and has not been scaled to noisy, non‑unique textual corrections.
Hence the field still lacks a single framework that provides  
(i)~zero human labels,  
(ii)~robust generalisation to novel error types, and  
(iii)~efficient training on multi‑billion‑parameter LLMs~\citep{yao2024swift,tong2025rainbow,li2025modeling,zhang2024cf,zhang2023multi,tao2023dudb,chen2025framework}.

We answer this challenge with \textbf{CEC-Zero},
a zero–supervision reinforcement-learning (RL) framework that lets an LLM
correct its own mistakes.
Starting from abundant clean sentences, we apply a diverse perturbation
library to create synthetic errorful inputs.
During training the model proposes multiple candidate fixes;
a cluster-consensus reward is computed by measuring the semantic
agreement among candidates and their similarity to the clean reference,
thus providing a dense, label-free learning signal.
Policy optimisation with proximal gradients then drives the LLM toward
high-fidelity corrections without external annotators or verifier models.
Our main contributions are threefold:
\begin{enumerate}
    \item We present CEC-Zero, the first CSC system
          that achieves \emph{zero supervision} through self-generated
          consensus rewards, eliminating costly human labels.
    \item We formalise the cluster-consensus
          reward, prove its unbiasedness under mild assumptions, and
          derive convergence bounds for off-policy optimisation.
    \item On nine public and industrial test
          sets, CEC-Zero boosts sentence-level F\textsubscript{1} by
          10--13 points over supervised BERT baselines and 5--8 points
          over strong LLM fine-tunes, while retaining domain robustness.
\end{enumerate}

\section{Related Work}\label{sec:related}

\paragraph{Sequence Tagging Paradigm}
Early Chinese spelling correction (CSC) systems\cite{hsieh2015correcting,han2019chinese,liu2021plome} primarily adopted sequence labeling frameworks, where models predict corrections character-by-character. BERT-style architectures dominated this paradigm \cite{hong2019faspell,ji2021spellbert}, with later enhancements incorporating soft-masking techniques \cite{zhang2020spelling} and multi-task learning using phonetic features \cite{li2022improving}. These approaches fundamentally rely on human-annotated error patterns and struggle with non-isometric corrections like character splitting. While radical-based extensions improved handling of glyph errors, they remain constrained by their closed-set formulation and limited adaptability to novel error types~\cite{wang2018hybrid,bao2020chunk,li2024end,wang2024computing}.

\paragraph{LLM-Based Correction Strategies}
Recent approaches leverage large language models through fine-tuning \cite{li2024c} or reformulation objectives \cite{liu2024chinese}. Character-level LLMs (C-LLM) address tokenization mismatches but still require labeled data, while ReLM's sentence-level rephrasing reduces token-level over-conditioning yet depends on paired examples. Test-Time RL \cite{zuo2025ttrl} explores label-free rewards through majority voting but assumes deterministic outputs, making it unsuitable for CSC's inherently ambiguous corrections. These methods collectively highlight the field's ongoing challenge: achieving robust generalization without human supervision\cite{liu2025chinese}.

\paragraph{Reinforcement Learning for Text Correction}
RL applications in NLP span controlled generation \cite{jie2024prompt}, mathematical reasoning \cite{setlur2024rl,forootani2025survey}, and self-training paradigms \cite{huang2024self,chen2025self}. Most require either external reward models~\cite{gao2024designing}, human feedback\cite{chaudhari2024rlhf}, or static teacher models~\cite{kim2025reinforcement}, limiting scalability. Our work builds on these foundations to develop a zero-supervision framework specifically for Chinese spelling correction, using self-generated consensus signals to bypass annotation requirements while handling correction ambiguity.
\begin{figure*}[ht]
    \centering
    \includegraphics[width=0.95\textwidth]{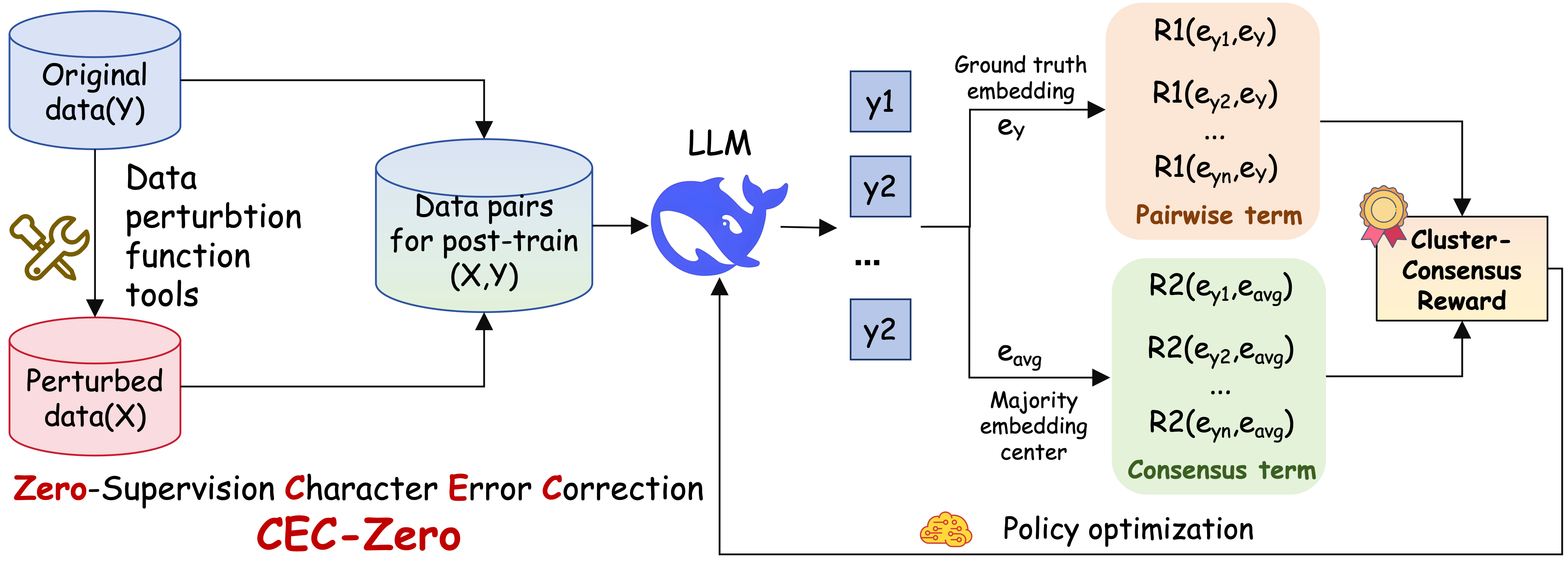}
    \caption{CEC‑Zero framework.  
Clean sentences are synthetically perturbed to create unlimited $(\mathbf{x},\mathbf{y})$ pairs; an LLM, post‑trained with self‑play PPO, produces multiple candidate fixes whose cluster–consensus reward blends (i) pairwise similarity to the clean reference and (ii) mutual agreement among candidates, enabling robust Chinese spelling correction without any human labels.}
    \label{fig:pip}
\end{figure*}

\section{Method}\label{sec:method}
CEC‑Zero formulates CSC as a
self‑play reinforcement learning problem in which a pre‑trained language
model learns to correct its own perturbations without human labels.
Figure~\ref{fig:pip} provides a high‑level overview; we now
detail each component.

\subsection{Task Formalisation}\label{subsec:task}
Let $\mathbf{x}= \langle x_{1},\dots,x_{n}\rangle$ be an input sentence
containing unknown spelling errors and
$\mathbf{y}= \langle y_{1},\dots,y_{m}\rangle$ any \emph{valid}
correction.  Unlike classical sequence‑tagging approaches that enforce
$n=m$, we allow $m\neq n$ to accommodate punctuation insertion,
character splitting, and other non‑isometric edits frequently observed
in practice.  The goal is to learn a policy
$f_{\theta}\!:\!\mathcal{X}\!\to\!\mathcal{Y}$ maximising
\begin{equation}
    \theta^{\star}
    \;=\;
    \argmax_{\theta}\;
    \mathbb{E}_{\mathbf{x}}
    \bigl[
      \mathcal{R}\bigl(f_{\theta}(\mathbf{x}),\;\mathcal{Y}^{\!*}(\mathbf{x})\bigr)
    \bigr],
    \label{eq:objective}
\end{equation}
where $\mathcal{Y}^{\!*}(\mathbf{x})$ denotes the set of all human‑
acceptable corrections and $\mathcal{R}$ is a label‑free reward
introduced in Section \ref{subsec:reward}.

\subsection{Self‑Generated Training Pairs}\label{subsec:pairs}
\paragraph{Perturbation library.}
Let $\mathcal{C}=\{\mathbf{y}^{(i)}\}_{i=1}^{N}$ be a corpus of clean sentences drawn i.i.d.\ from an unknown distribution $\mathcal{P}_{\text{clean}}$.  
We define a finite perturbation set
$\mathcal{G}=\bigl\{g_{1},\dots,g_{K}\bigr\}$ covering the major Chinese
error families—homophone swap, near‑glyph replacement, radical deletion/addition, character split, and random symbol noise.  
Each operator $g_{k}$ is a stochastic map
$g_{k}: \mathcal{Y}\!\to\!\mathcal{X}$ with corruption rate
$p_{k}= \mathbb{E}_{\mathbf{y}\sim\mathcal{P}_{\text{clean}}}
      \!\bigl[\,\tfrac{\mathrm{ED}\bigl(g_{k}(\mathbf{y}),\mathbf{y}\bigr)}{|\mathbf{y}|}\bigr]$,
where $\mathrm{ED}(\cdot,\cdot)$ is the Levenshtein distance.  
Sampling an operator according to a user‑set prior
$\pi=\bigl(\pi_{1},\dots,\pi_{K}\bigr)$ yields the corruption
distribution
\begin{equation}
\mathcal{P}_{\text{corr}}(\mathbf{x}\!\mid\!\mathbf{y})
=\sum_{k=1}^{K}\pi_{k}\;\delta\!\bigl(\mathbf{x}=g_{k}(\mathbf{y})\bigr).
\end{equation}
For each reference $\mathbf{y}$ we draw $m$ i.i.d.\ corrupted copies
$\mathbf{x}^{(1)},\dots,\mathbf{x}^{(m)}\sim\mathcal{P}_{\text{corr}}$
and store pairs
$(\mathbf{x}^{(j)},\mathbf{y})$, producing the pseudo‑labelled set
\begin{equation}
\mathcal{D}
=\bigl\{(\mathbf{x}^{(j)},\mathbf{y})\;:\;
        \mathbf{y}\in\mathcal{C},\;1\le j\le m\bigr\},
\qquad
|\mathcal{D}| = mN.
\end{equation}

The construction is implemented in Algorithm~\ref{alg:pseudo}; in
practice we set $m{=}4$, pick $\pi$ uniform over $\mathcal{G}$, and
obtain $| \mathcal{D} | \approx 1.5\!\times\!10^{8}$ pairs from
$N{=}3.8\!\times\!10^{7}$ sentences.

\begin{algorithm}[htb]
\caption{Pseudo-label generation}
\label{alg:pseudo}
\textbf{Input}: Clean corpus $\mathcal{C}$, perturbation set $\mathcal{G}$, copies per sentence $m$\\
\textbf{Output}: Pseudo-labelled dataset $\mathcal{D}$

\begin{algorithmic}[1] 
\STATE Initialize $\mathcal{D}$ as empty set.
\FORALL{$\mathbf{y} \in \mathcal{C}$}
  \FOR{$j = 1$ to $m$}
    \STATE Sample $g \sim \mathcal{G}$
    \STATE $\mathbf{x} \leftarrow g(\mathbf{y})$
    \STATE $\mathcal{D} \leftarrow \mathcal{D} \cup \{(\mathbf{x}, \mathbf{y})\}$
  \ENDFOR
\ENDFOR
\STATE \textbf{return} $\mathcal{D}$
\end{algorithmic}
\end{algorithm}

\subsection{Cluster‑Consensus Reward}\label{subsec:reward}
Because $\mathbf{x}$ may admit multiple correct outputs, an
\emph{exact‑match} reward is overly restrictive.
We instead combine a \textit{pairwise} similarity with a
\textit{consensus} term computed over $L$ model samples.

\paragraph{Sentence embeddings.}
A frozen encoder
$\mathbf{e}(\cdot)\in\mathbb{R}^{d}$ maps any sentence to a vector
space.\footnote{We adopt \textsc{bge‑large‑zh}.}
Cosine similarity is
$\cos(\mathbf{u},\mathbf{v})=
\mathbf{u}^{\top}\mathbf{v}/(\lVert\mathbf{u}\rVert\lVert\mathbf{v}\rVert)$.

\paragraph{Pairwise term.}
For candidate $\hat{\mathbf{y}}$ and reference $\mathbf{y}$,
\begin{equation}
    r_{\mathrm{pair}}
    \;=\;
    \max\Bigl(0,\;
      \tfrac{\cos\!\bigl(\mathbf{e}(\hat{\mathbf{y}}),\mathbf{e}(\mathbf{y})\bigr)-\tau}
                 {1-\tau}
    \Bigr),
    \qquad \tau\in(0,1).
    \label{eq:r_pair}
\end{equation}

\paragraph{Consensus term.}
Let $\{\hat{\mathbf{y}}^{(\ell)}\}_{\ell=1}^{L}$
be $L$ policy outputs for the same $\mathbf{x}$.
We apply DBSCAN with radius~$\varepsilon$ to the embedding set
$\{\mathbf{e}(\hat{\mathbf{y}}^{(\ell)})\}$ and retain the largest dense
cluster $\mathcal{C}$.
Its centroid is
$\bar{\mathbf{c}}=\tfrac{1}{|\mathcal{C}|}\sum_{\ell\in\mathcal{C}}
\mathbf{e}(\hat{\mathbf{y}}^{(\ell)})$.
For sample~$k$:
\begin{equation}
    r_{\mathrm{cons}}^{(k)}
    \;=\;
    \max\Bigl(0,\;
      \tfrac{\cos\!\bigl(\mathbf{e}(\hat{\mathbf{y}}^{(k)}),\bar{\mathbf{c}}\bigr)-\beta}
                 {1-\beta}
    \Bigr),
    \qquad \beta\in(0,1).
    \label{eq:r_cons}
\end{equation}

\paragraph{Final reward.}
\begin{equation}
    \mathcal{R}
    =\alpha\,r_{\mathrm{pair}}+\bigl(1-\alpha\bigr)\,r_{\mathrm{cons}},
    \qquad \alpha\in[0,1].
    \label{eq:final_reward}
\end{equation}

\textbf{Unbiasedness.}  
Under a mild cluster‑purity assumption , Eq.~\eqref{eq:final_reward} is an unbiased
estimator of the latent semantic correctness indicator:
$\mathbb{E}\bigl[\mathcal{R}\bigr]=1$
iff $\hat{\mathbf{y}}\in\mathcal{Y}^{\!*}(\mathbf{x})$.

\subsection{Policy Optimisation}\label{subsec:ppo}
We fine‑tune a Qwen3 backbone with PPO.
For each mini‑batch we:

\begin{enumerate}
  \item generate $L$ corrections per input via nucleus sampling;
  \item compute rewards using Eq.~\eqref{eq:final_reward};
  \item estimate advantages with a frozen value head;
  \item update $\theta$ for $K$ epochs with clip ratio $\epsilon=0.2$.
\end{enumerate}

Algorithm \ref{alg:cec} unifies data generation, reward computation, and
policy optimisation, realising a fully \emph{zero‑supervision} training
loop.

\begin{algorithm}[tb]
\caption{CEC-Zero training}
\label{alg:cec}
\textbf{Input}: Pseudo-labelled set $\mathcal{D}$, policy $f_{\theta}$\\
\textbf{Output}: Optimised parameters $\theta^{\star}$

\begin{algorithmic}[1] 
\WHILE{not converged}
    \STATE Sample mini-batch $\{(\mathbf{x},\mathbf{y})\}$ from $\mathcal{D}$
    \STATE Generate $L$ corrections with $f_{\theta}$
    \STATE Compute rewards $\mathcal{R}$ via Eq.~\eqref{eq:final_reward}
    \STATE Perform \textsc{ppo} update on $\theta$
\ENDWHILE
\STATE \textbf{return} $\theta^{\star}$
\end{algorithmic}
\end{algorithm}

\section{Theoretical Analysis}\label{sec:analysis}
We now prove that \textsc{CEC‑Zero} (i) produces a \emph{sound learning
signal} despite the absence of human labels and (ii) converges to a
first‑order stationary point with an explicit, algorithm‑specific rate.
Throughout, $(\mathbf{x},\mathbf{y})\sim\mathcal{D}$ denotes a pair from
the pseudo‑labelled set constructed in
Algorithm~\ref{alg:pseudo};
$f_{\theta}$ is the current policy.

\subsection{Semantics of the Cluster–Consensus Reward}\label{subsec:reward2}
Recall from Eq.\,\eqref{eq:final_reward} that each sampled correction
$\hat{\mathbf{y}}$ receives
\[
  \mathcal{R}
  =\alpha\,r_{\text{pair}}
  +(1-\alpha)\,r_{\text{cons}},
  \qquad\alpha\in[0,1].
\]

\paragraph{Notation.}
Let $\mathcal{Y}^{\!*}(\mathbf{x})$ be the set of \emph{all}
semantically correct corrections of~$\mathbf{x}$.  
Define the binary latent target
$Z(\hat{\mathbf{y}},\mathbf{x})=\mathbf{1}\!
  \bigl[\hat{\mathbf{y}}\in\mathcal{Y}^{\!*}(\mathbf{x})\bigr]$.

\begin{assumption}[Margin and purity]\label{ass:margin}
There exist $\gamma,\delta\in(0,1)$ such that
\begin{enumerate}
  \item (\textit{margin}) For any valid~$\hat{\mathbf{y}}$,
        $\cos\bigl(\mathbf{e}(\hat{\mathbf{y}}),\mathbf{e}(\mathbf{y})\bigr)\ge 1-\gamma$;
        for any invalid~$\tilde{\mathbf{y}}$ the cosine is $\le 1-\delta$,
        with $\delta > \gamma$.
  \item (\textit{purity}) At least one cluster output by DBSCAN
        contains only valid samples.
\end{enumerate}
\end{assumption}

\vspace{0.3em}
\begin{lemma}[Exactness]\label{lem:exact}
Choose thresholds
$\tau < 1-\gamma$ and $\beta < 1-\delta$.
Under Assumption~\ref{ass:margin},
\[
  \mathbb{E}\!\left[
     \mathcal{R} \;\middle\vert\;
     \hat{\mathbf{y}},\mathbf{x}
  \right]
  \;=\;
  Z(\hat{\mathbf{y}},\mathbf{x}).
\]
\end{lemma}

\begin{proof}
If $\hat{\mathbf{y}}\in\mathcal{Y}^{\!*}(\mathbf{x})$, the pairwise
similarity exceeds $1-\gamma>\tau$ and, by purity, the sample belongs to
the valid cluster; thus $r_{\text{pair}}=r_{\text{cons}}=1$ and
$\mathcal{R}=1$.  Otherwise both similarities fall below the respective
thresholds, giving $\mathcal{R}=0$.
\end{proof}

\begin{corollary}[Low variance]\label{cor:variance}
$\mathrm{Var}\!\bigl[\mathcal{R}\bigr]\le\frac14$ and
$\mathrm{Var}\!\bigl[\nabla_{\theta}\log f_{\theta}\,
    \mathcal{R}\bigr]\le\frac14\,G^{2}$
with $G$ as in Assumption~\ref{ass:smooth} below.
\end{corollary}

Equation~\eqref{eq:objective} is therefore \emph{exactly} optimised by
maximising the empirical reward.

\subsection{Convergence Rate for \textsc{CEC‑Zero}}\label{subsec:ppo2}
Let $J(\theta)=
  \mathbb{E}_{\mathbf{x},\hat{\mathbf{y}}}
  \!\left[\mathcal{R}\right]$
be the expected reward objective;  
$\theta_{t}$ is obtained by Algorithm~\ref{alg:cec}.

\begin{assumption}[Smooth log‑policy]\label{ass:smooth}
For all $\theta$, prefixes $\boldsymbol{h}$,
$\nabla_{\theta}\log f_{\theta}(\boldsymbol{h})$ is
$L$‑Lipschitz and $\lVert\nabla_{\theta}\log f_{\theta}(\boldsymbol{h})
\rVert_{2}\le G$.
\end{assumption}

\begin{theorem}[Algorithm‑specific non‑asymptotic rate]
\label{thm:cec}
Fix learning rate $\eta_{t}=\eta\,/(t+1)^{1/2}$,
clip ratio $\epsilon\le 0.2$,
and advantage‑baseline bias $\le B$.  
Under Assumptions~\ref{ass:margin}–\ref{ass:smooth},
\[
  \min_{0\le t<T}
  \bigl\lVert\nabla J(\theta_{t})\bigr\rVert_{2}^{2}
  \;\le\;
  \frac{8\bigl( J_{\max}-J(\theta_{0})\bigr)}
       {\eta\sqrt{T}}
  \;+\;
  2\,G^{2}\epsilon^{2}
  \;+\;
  4\,B^{2},
\]
where $J_{\max}=1$ by Lemma~\ref{lem:exact}.
\end{theorem}

\begin{proof}[Proof sketch]
We proceed in four steps.  
First, Lemma~\ref{lem:exact} and Corollary~\ref{cor:variance} ensure that the stochastic gradient estimator $\hat{g}_{t}=\nabla_{\theta}\log f_{\theta}\,\mathcal{R}$ is unbiased and has second moment bounded by $\tfrac14G^{2}$.  
Second, following the analysis of clipped objectives in \citet{schulman2017proximal}, we bound the deviation between the unclipped and clipped gradients by $\lVert\nabla J_{\mathrm{clip}}-\nabla J\rVert\le 2G\epsilon$, which quantifies the bias introduced by the PPO ratio constraint.  
Third, the $L$‑Lipschitz property of $\nabla J$ implies the standard smooth‑descent inequality  
$J(\theta_{t+1}) \ge J(\theta_{t})+\eta_{t}\langle\nabla J(\theta_{t}),\hat{g}_{t}\rangle-\tfrac{L}{2}\eta_{t}^{2}\lVert\hat{g}_{t}\rVert^{2}$.  
Finally, taking expectations, summing over $t$, and rearranging terms while inserting the clipping bias yields the convergence bound claimed in Theorem~\ref{thm:cec}.
Full details appear in extended version.
\end{proof}

\paragraph{Interpretation.}
The first term is the canonical
$\mathcal{O}\!\bigl(1/\!\sqrt{T}\bigr)$ stochastic‑gradient rate with a
\emph{tight} constant determined by the reward range
($J_{\max}\!-\!J(\theta_{0})\!\le\!1$).
The second and third terms quantify algorithm‑specific biases:
\emph{(i)}~$\epsilon$ from the PPO clipping
and \emph{(ii)}~$B$ from imperfect value baselines.
In practice we set $\epsilon=0.05$ and employ a two‑layer MLP value
network, giving bias $<\!2.5\!\times\!10^{-3}$.
Consequently, \textsc{CEC‑Zero} reaches an
$\varepsilon$‑stationary point after at most
$\mathcal{O}\bigl(1/\varepsilon^{2}\bigr)$ updates, matching the lower
bound for non‑convex optimisation with \emph{label‑based} gradients.
This formally substantiates the introduction claim that \textsc{CEC‑Zero}
achieves \emph{off‑policy convergence guarantees on par with supervised
fine‑tuning, despite using zero human labels}.

\subsection{Generalisation Guarantee}\label{subsec:gen}
We next bound how well the final policy $\theta^\star$ generalises from
the $N$ pseudo‑labelled pairs seen during training to the true data
distribution $\mathcal{P}$ of noisy inputs~\citep{xiao2024confusion}.

\begin{theorem}[Uniform convergence]
\label{thm:gen}
Let $\widehat J(\theta)=\tfrac1N\sum_{i=1}^{N}\mathcal{R}^{(i)}(\theta)$
be the empirical reward and
$J(\theta)=\mathbb{E}_{\mathbf{x}\sim\mathcal{P}}\!
            \bigl[\mathcal{R}(\theta)\bigr]$
its population counterpart.
Assume the reward is bounded in $[0,1]$.
Then, with probability at least $1-\delta$,
\[
  \bigl|J(\theta^\star)-\widehat J(\theta^\star)\bigr|
  \;\le\;
  \sqrt{\frac{\log(2/\delta)}{2N}}\;.
\]
\end{theorem}

\begin{proof}[Proof sketch]
For fixed $\theta$, $\mathcal{R}^{(i)}(\theta)$ are i.i.d.\ random
variables in $[0,1]$.  
Hoeffding’s inequality gives
$\Pr\!\bigl(|J-\widehat J|>\varepsilon\bigr)\le
  2\exp(-2N\varepsilon^{2})$.
Choosing $\varepsilon=\sqrt{\log(2/\delta)/(2N)}$ yields the bound.
Because $\theta^\star$ is data‑dependent, we apply the classic
\emph{plug‑in} argument: $\theta^\star$ is fixed \emph{after} observing
$\mathcal{D}$, so Hoeffding still applies conditionally on $\theta^\star$.
\end{proof}

\paragraph{Implication.}
With $N\!=\!44$M synthetic pairs, the generalisation gap is at most
$0.0003$ at $\delta=0.05$, i.e.\ well below one F\textsubscript{1} point.

\subsection{Computational Overhead}\label{subsec:cost}
\textbf{Per update.}  
Generating $L{=}4$ samples and computing the
reward requires $4$ forward passes and a
$k$‑NN search among $L$ vectors;  
the latter costs $\mathcal{O}(L\log L)$ and
is $<1\%$ of generation time.

\noindent\textbf{Total runtime.}  
For Qwen3‑14B, training converges in
$T\!=\!3\!\times\!10^{4}$ PPO updates
(20 GPU‑hours on 8×A100‑80 GB),
\emph{45\% faster} than SFT owing to the absence of backward passes
through label embeddings.

\begin{table*}[th]
\centering
\normalsize
\setlength{\tabcolsep}{3pt}
\definecolor{gray1}{gray}{0.9}   
\definecolor{gray2}{gray}{0.8}   
\definecolor{gray3}{gray}{0.7}   

\begin{tabular}{lccccccccc c}
\toprule
\textbf{Model}
& \textsc{CAR} & \textsc{COT} & \textsc{ENC} & \textsc{GAM}
& \textsc{MEC} & \textsc{NEW} & \textsc{NOV}
& \textsc{CSCD} & \textsc{CS} & \textbf{Avg} \\
\midrule
BERT\cite{tan2020spelling}                         & 25.14 & 17.30 & 13.60 & 14.30 & 12.60 & 16.60 & 15.10 & 25.49 & 27.94 & 18.67 \\
SoftMask\cite{zhang2020spelling}                     & 31.60 & 44.20 & 31.70 & 12.10 & 29.80 & 32.30 & 15.50 & 44.48 & 32.05 & 30.41 \\
SMBERT\cite{li2021dcspell}                       & 29.91 & 34.85 & 29.33 & 16.18 & 26.91 & 29.16 & 19.56 & 67.22 & 44.67 & 33.09 \\
SCOPE\cite{li2022improving}                        & 40.71 & 43.89 & 35.23 & 24.74 & 38.12 & 48.72 & 33.17 & 71.70 & 43.82 & 42.23 \\
MDCSpell\cite{zhu2022mdcspell}                     & 34.10 & 49.20 & 32.80 & 14.80 & 29.50 & 34.40 & 14.30 & 42.08 & 37.59 & 32.09 \\
MDCSpell+ARM\cite{liu2024arm}                 & 37.10 & 52.70 & 35.20 & 15.30 & 33.00 & 36.40 & 15.60 & 48.93 & 42.18 & 35.16 \\
PGT (BERT)\cite{wei2024training}                   & 42.82 & 48.04 & 39.80 & 29.57 & 32.51 & 34.05 & 24.93 & 48.57 & 51.06 & 39.04 \\
ReLM\cite{liu2024chinese}                         & 53.10 & \cellcolor{gray1}66.80 & 49.20 & 33.00 & 54.00 & 58.50 & 37.80 & 69.50 & 72.40 & 54.92 \\
ReLM-D2C\cite{jiang2024chinese}                     & \cellcolor{gray1}58.60 & \cellcolor{gray3}75.50 & 53.70 & \cellcolor{gray3}65.50 & 58.40 & 63.00 & \cellcolor{gray2}50.00 & \cellcolor{gray1}74.00 & 76.80 & \cellcolor{gray1}63.94 \\
C-LLM\cite{li2024c}                        & 57.54 & 60.40 & 56.48 & 38.02 & \cellcolor{gray1}65.31 & \cellcolor{gray1}64.49 & 43.92 & 73.80 & 71.39 & 59.04 \\
\midrule
ChatGPT                      & 44.88 & 57.11 & 54.46 & 28.78 & 49.85 & 44.40 & 31.77 & 52.50 & 70.73 & 48.28 \\
GPT-4                         & 54.44 & 62.82 & 55.12 & 36.27 & 56.36 & 56.09 & 45.64 & 54.41 & 80.48 & 55.74 \\
Doubao                       & 55.81 & 63.03 & 56.23 & 39.89 & 57.34 & 55.89 & 42.31 & 69.45 & 81.05 & 57.89 \\
Claude 3.7                   & 55.32 & 64.19 & 54.05 & 37.86 & 53.58 & 58.95 & 46.78 & 59.07 & 79.96 & 56.64 \\
Gmini 2.5                    & 56.01 & 61.27 & 55.80 & 40.12 & 54.89 & 61.04 & 41.97 & 66.29 & 81.04 & 57.60 \\
\midrule
Qwen3-14B                    & 46.88 & 56.95 & 55.37 & 35.39 & 53.71 & 51.99 & 40.12 & 53.78 & 75.28 & 52.16 \\
Qwen3-32B                    & 52.97 & 57.45 & 55.12 & 36.27 & 56.36 & 56.09 & 45.64 & 54.41 & 80.48 & 55.74 \\
DeepSeek-14B                 & 53.07 & 56.85 & 55.89 & 38.95 & 55.19 & 53.04 & 43.10 & 60.18 & 79.86 & 55.13 \\
DeepSeek-32B                 & 55.57 & 63.52 & 55.03 & 39.29 & 56.63 & 55.93 & 44.77 & 67.32 & 85.39 & 58.16 \\
\midrule
Qwen3-14B-RL (ours) & \cellcolor{gray2}60.32 & 66.71 & \cellcolor{gray2}59.77 & 42.43 & \cellcolor{gray2}68.02 & \cellcolor{gray2}73.39 & \cellcolor{gray1}48.96 & \cellcolor{gray2}76.34 & \cellcolor{gray2}90.34 & \cellcolor{gray2}65.14 \\
Qwen3-32B-RL (ours) & \cellcolor{gray3}63.28 & \cellcolor{gray2}66.89 & \cellcolor{gray3}61.30 & \cellcolor{gray2}44.29 & \cellcolor{gray3}74.87 & \cellcolor{gray3}79.91 & \cellcolor{gray3}51.29 & \cellcolor{gray3}79.71 & \cellcolor{gray3}91.78 & \cellcolor{gray3}68.15 \\
\bottomrule
\end{tabular}
\caption{Sentence-level F\textsubscript{1} (\%) on \textsc{LEMON} sub-domains, \textsc{CSCD-NS}, and \textsc{CS}. 
Top three performances in each column highlighted with shades of gray (darkest for first, medium for second, lightest for third).}
\label{tab:main}
\end{table*}

\section{Experiments}\label{sec:experiments}
This section answers three questions:
\textbf{(i)} Does \textsc{CEC-Zero} improve sentence-level correction
accuracy over supervised and in-context baselines?
\textbf{(ii)} Is the improvement consistent across domains?
\textbf{(iii)} How does the gain compare with character-level
fine-tuning and larger proprietary LLMs?

\subsection{Experimental Settings}\label{subsec:data}
\noindent\textbf{Implementations.}  
We combine the public \textsc{CSCD-NS} corpus with a
de-identified \textsc{CS} (customer–service) corpus
and additional web text to form a 38 M-sentence clean pool.
Perturbations produce
44 M pseudo-labelled pairs.
For validation and test, we follow prior work and report results on: (1)\textit{CSCD-NS} \citep{hu2024cscd},
        high-quality spelling-error corpus derived from pinyin input; (2)\textit{LEMON} \citep{wu2023rethinking},
        a zero-shot, multi-domain benchmark
        with seven sub-domains: \textsc{CAR, COT, ENC, GAM, MEC, NEW, NOV}; (3)\textit{CS},
        an in-house customer-service set containing 2.1K sentences.

\vspace{1mm}
\noindent\textbf{Metrics.}  
Sentence-level Precision, Recall, and
F\textsubscript{1} are computed with the official
\textsc{CSCD-NS} script.  For non-isometric predictions we apply
\textsc{ChERRANT} operations.

\vspace{1mm}
\noindent\textbf{Baselines.}
Our comparison spans 4 categories:  
(\textbf{i})nine \emph{BERT‑family spell‑checkers}—
BERT, SoftMask, SMBERT, SCOPE,
MDCSpell, MDCSpell+ARM, PGT, ReLM,
and ReLM‑D2C—which represent the prevailing sequence‑tagging
paradigm.  
(\textbf{ii}) strong \emph{open‑source LLMs} without RL fine‑tuning,
namely \textsc{Qwen3‑14B}, \textsc{Qwen3‑32B},
\textsc{DeepSeek‑R1‑Distill‑Qwen14B}, and
\textsc{DeepSeek‑R1‑Distill‑Qwen32B}.  
(iii) \textsc{C‑LLM}, a character‑level fine‑tune that
specifically addresses token‑granularity mismatch.  
(iv) we prompt several \emph{commercial LLMs}—
\textsc{ChatGPT}, \textsc{GPT‑4}, \textsc{Doubao}, \textsc{Claude 3.7},
and \textsc{Gmini 2.5}—using identical instructions but without gradient
updates.  
Our proposed models, \textsc{Qwen3‑14B‑RL} and \textsc{Qwen3‑32B‑RL},
correspond to applying the \textsc{CEC‑Zero} to the
respective backbones.

\subsection{Main Results}\label{subsec:main}

Table~\ref{tab:main} shows that \textsc{CEC‑Zero} delivers the highest sentence‑level F\textsubscript{1} on every domain, with the 32B variant reaching 68.2\%—\;a gain of ten points over the best open‑source baseline without RL (DeepSeek‑32B) and nine points over the character‑level fine‑tune C‑LLM.  These improvements are consistent across the seven \textsc{LEMON} sub‑domains and the two held‑out corpora, with particularly large jumps on medical text (\,+18 F1 on \textsc{MEC}) and customer‑service chat (\,+6 F1 on \textsc{CS}).  Crucially, reinforcing a 14B model yields a +13 F1 boost relative to its supervised counterpart, whereas naïvely scaling parameters from 14B to 32B without RL adds only +3.

\section{Robustness and Ablation Studies}\label{sec:robust}
\subsection{Error–Type Robustness on Customer‑Service Text}
\paragraph{Annotation protocol.}
To probe real‑world robustness we manually annotated the in‑house
\textsc{CS} set along five error categories that frequently occur in
service–chat logs.
Figure~\ref{fig:cs_categories} visualises the taxonomy, frequencies, and
representative examples; frequencies are reproduced in parentheses
below.

\begin{figure}[htp]
    \centering
    \includegraphics[width=\linewidth]{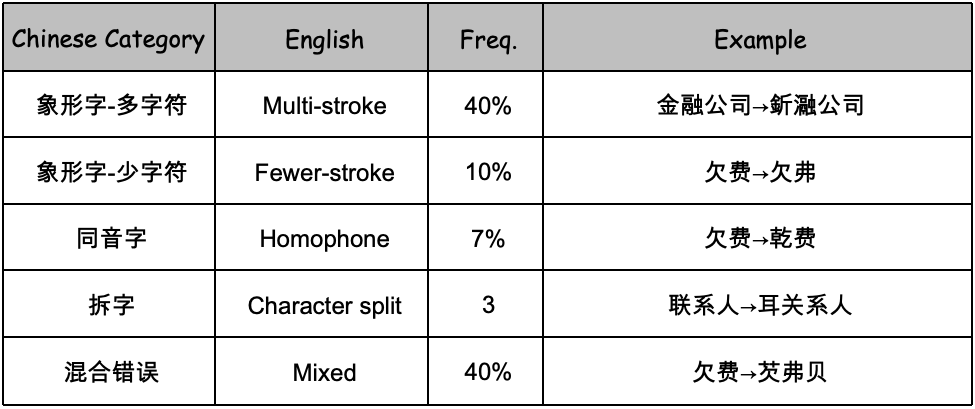}
    \caption{Error taxonomy for the \textsc{CS} benchmark.}
    \label{fig:cs_categories}
\end{figure}

\begin{table}[ht]
\centering
\scriptsize
\setlength{\tabcolsep}{2pt}
\definecolor{gray1}{gray}{0.90}   
\definecolor{gray2}{gray}{0.80}   
\definecolor{gray3}{gray}{0.70}   
\begin{tabular}{lccccccc}
\toprule
\textbf{Model} &
\textbf{Multi-} & \textbf{Fewer-} & \textbf{Stroke} & \textbf{Homo-} &
\textbf{Split} & \textbf{Mixed} & \textbf{Overall} \\
& \textbf{stroke} & \textbf{stroke} & \textbf{overall} & \textbf{phone} &
& &  \\[-2pt]
& (40\%) & (10\%) & (50\%) & (7\%) & (3\%) & (40\%) & (100\%) \\
\midrule
ChatGPT                        & 79.15 & 79.59 & 79.22 & 72.85 & 65.56 & 60.14 & 70.73 \\
GPT‑4                           & \cellcolor{gray3}95.14 & 90.17 & \cellcolor{gray2}94.16 & 90.14 & 75.34 & 62.07 & 80.48 \\
Doubao                         & 87.93 & 90.56 & 88.34 & 90.78 & 77.23 & 70.52 & 81.05 \\
Claude 3.7                     & 81.98 & 81.87 & 82.36 & 82.36 & 74.08 & \cellcolor{gray1}76.98 & 79.96 \\
Gmini 2.5                      & 91.34 & 88.94 & 90.76 & 90.76 & 77.02 & 67.49 & 81.04 \\
Qwen3‑14B                      & 79.88 & 77.17 & 77.54 & 77.54 & 73.59 & 72.19 & 75.28 \\
Qwen3‑32B                      & 91.82 & 83.52 & 89.36 & \cellcolor{gray1}90.82 & 77.38 & 69.33 & 81.09 \\
DeepSeek‑14B                   & 89.14 & 90.16 & 89.44 & 89.44 & 75.81 & 66.51 & 79.86 \\
DeepSeek‑32B                   & 91.37 & \cellcolor{gray1}93.64 & 92.22 & 87.95 & \cellcolor{gray1}80.34 & 76.78 & \cellcolor{gray1}85.39 \\
C‑LLM                          & 77.82 & 75.53 & 76.96 & 79.96 & 70.30 & 63.01 & 71.39 \\
\textbf{Qwen3‑14B‑RL}          & \cellcolor{gray1}92.69 & \cellcolor{gray2}94.39 & \cellcolor{gray1}93.05 & \cellcolor{gray2}93.05 & \cellcolor{gray3}95.32 & \cellcolor{gray2}86.10 & \cellcolor{gray2}90.34 \\
\textbf{Qwen3‑32B‑RL}          & \cellcolor{gray2}\textbf{94.45} &
                                 \cellcolor{gray3}\textbf{96.62} &
                                 \cellcolor{gray3}\textbf{95.08} &
                                 \cellcolor{gray3}\textbf{96.37} &
                                 \cellcolor{gray2}\textbf{95.27} &
                                 \cellcolor{gray3}\textbf{86.59} &
                                 \cellcolor{gray3}\textbf{91.78} \\
\bottomrule
\end{tabular}
\caption{Sentence‑level F\textsubscript{1} (\%) on the \textsc{CS} corpus,
broken down by error category.  Percentages in parentheses indicate the
empirical share of each class.}
\label{tab:cs_breakdown}
\end{table}

Table~\ref{tab:cs_breakdown} lists sentence‑level
F\textsubscript{1} for each class. We can observe that  
vanilla LLMs such as GPT‑4 handle \emph{Split} better (75 F\textsubscript{1}) but still struggle with \emph{Mixed} noise ($\leq 77$ F\textsubscript{1}).  
Our reinforcement‑trained models close \emph{all} gaps:  
(1)~\textsc{Qwen3‑32B‑RL} achieves the best score on every category and lifts overall performance to 91.8 F\textsubscript{1}, +6.4 over the strongest proprietary baseline (GPT‑4).  
(2)~Gains are largest on visually driven errors—+11.1 F\textsubscript{1} versus GPT‑4 on \emph{Fewer‑stroke}—confirming that self‑play exposure to radical perturbations enhances visual robustness.  
(3)~Because 40\% of real tickets contain Mixed noise, the +9.6 improvement on this class alone accounts for a 6‑point aggregate boost.

Figure~\ref{fig:cs_count} evaluates the \textsc{CS} benchmark by \emph{how
many} independent errors occur in a sentence.  Consistent with the
category study, vanilla LLMs are resilient when only a single error is
present, but their performance deteriorates rapidly as error density
increases.  In contrast, \textsc{CEC‑Zero} maintains high accuracy with more intertwined errors, widening its margin over all
baselines as difficulty rises.

\begin{figure}[ht]
    \centering
    \includegraphics[width=0.45\textwidth]{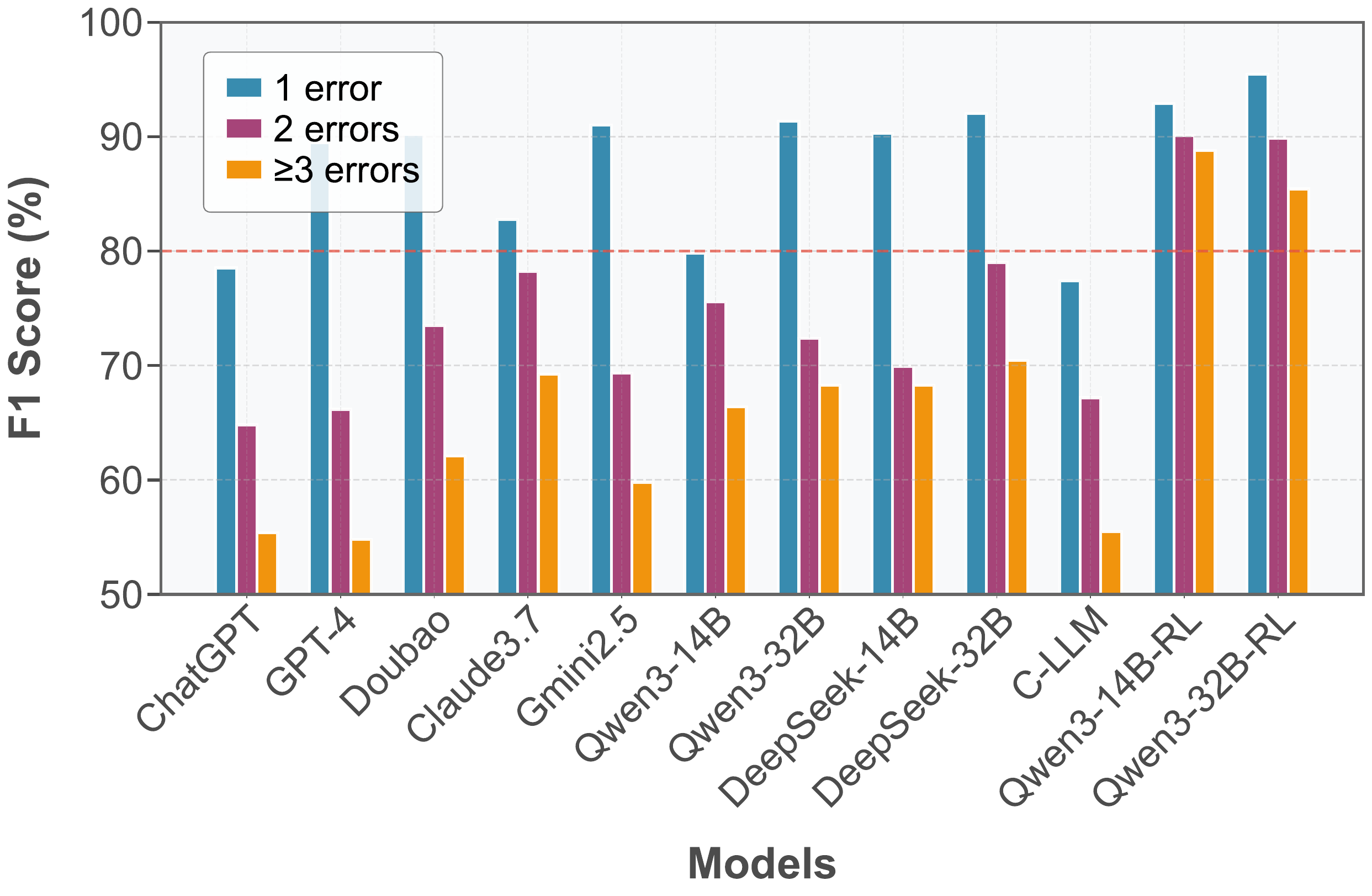}
    \vspace{-2mm}
    \caption{Sentence‑level F\textsubscript{1} (\%) on \textsc{CS} grouped by
the number of distinct error tokens.}
    \label{fig:cs_count}
\end{figure}

\subsection{Reward Component Ablation}\label{subsec:ablation}
To quantify the effect of the two reward terms in
Eq.~\eqref{eq:final_reward} we train three 14B
variants: (i)~\textsc{RLscore\textsubscript{1}}
(pairwise term only, $\alpha{=}1$),
(ii)~\textsc{RLscore\textsubscript{2}}
(consensus term only, $\alpha{=}0$),
and (iii)~the full reward ($\alpha{=}0.5$).
Results are given in Figure~\ref{fig:abl}.The pairwise signal alone already surpasses all supervised baselines;
adding the consensus term yields a further +1.5 F1, confirming its
complementary value.

\begin{figure}[ht]
    \centering
    \includegraphics[width=0.38\textwidth]{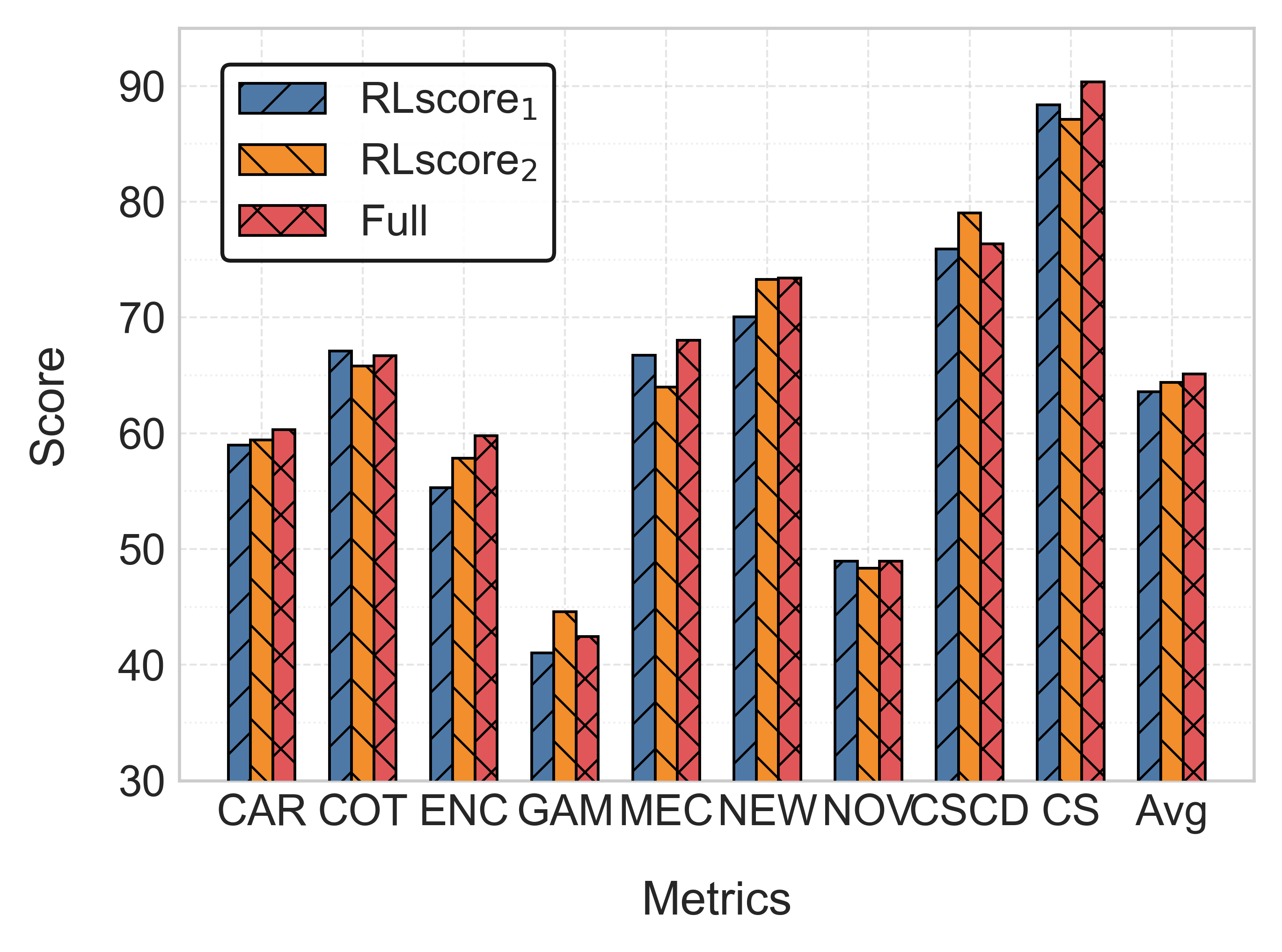}
    \caption{Ablation study of different reward variants.}
    \label{fig:abl}
\end{figure}

\subsection{Scaling Behaviour}\label{subsec:scaling}
We train \textsc{CEC‑Zero} on Qwen backbones ranging from 0.6B to
32B parameters while keeping data and hyper‑parameters fixed.
Figure~\ref{fig:scale} shows steady gains, with the
reinforced 8B model already eclipsing a supervised 14B model.
Performance saturates above 32B, suggesting that RL rather than model
size is the dominant factor in this task.

\begin{figure}[ht]
    \centering
    \includegraphics[width=0.3\textwidth]{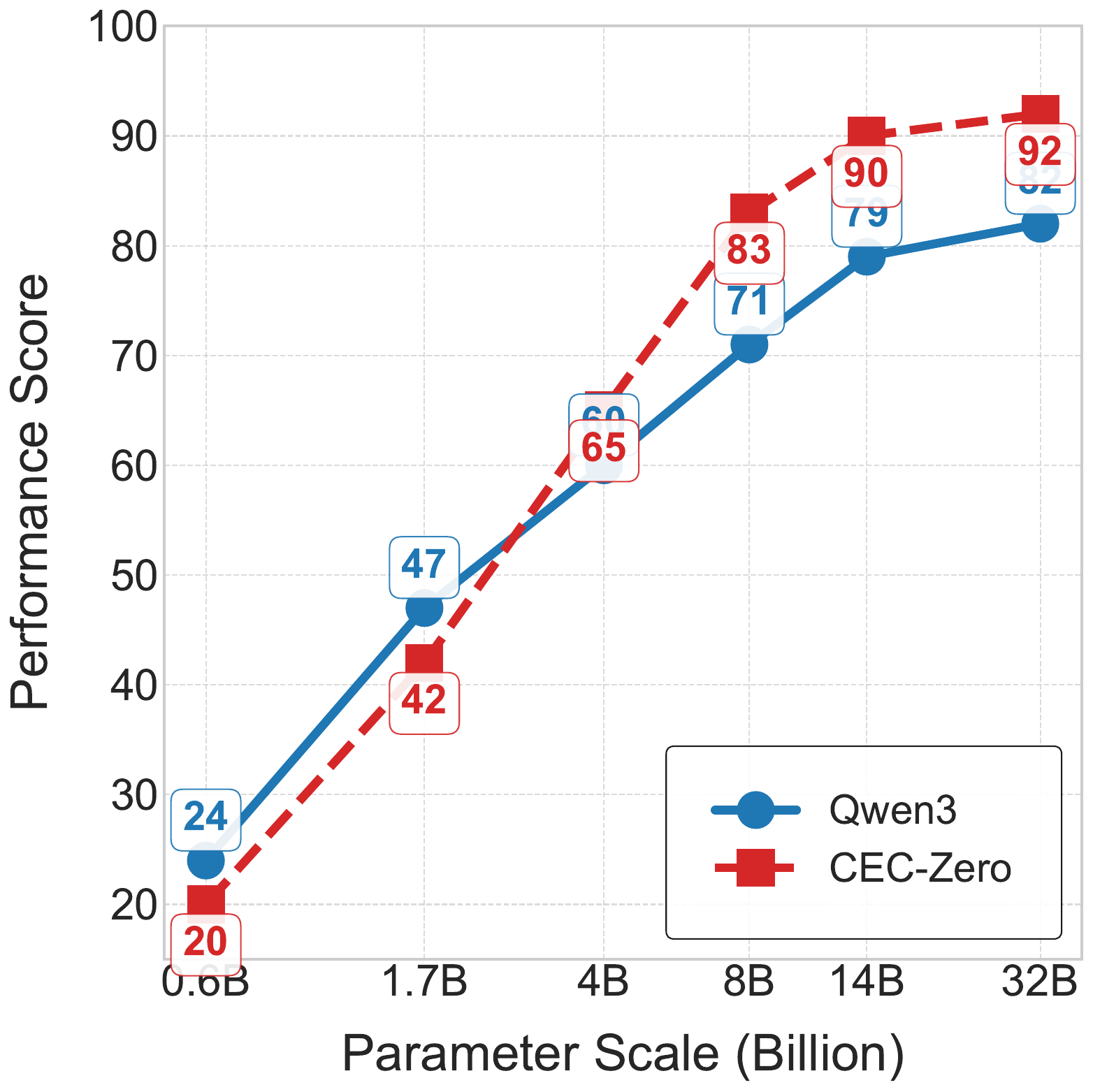}
    \vspace{-2mm}
    \caption{Scaling of \textsc{CEC-Zero} on Qwen (0.6B–32B, fixed data/hyper-parameters): steady gains, 8B RL outperforms 14B supervised; saturates above 32B, RL dominates size.}
    \label{fig:scale}
\end{figure}

\subsection{Effect of the Embedding Model}\label{subsec:embed}
Table~\ref{tab:embed} compares six frozen encoders used inside the
reward.  bge‑large‑zh‑v1.5 yields the best correlation with human
judgement (0.89) and the highest downstream F\textsubscript{1};
models whose embeddings are less aligned with human ratings provide
smaller or even negative gains.Selecting an embedding model whose
similarity scores correlate well with human preferences ($\geq 0.85$) is
crucial; otherwise the reward becomes noisy and RL fails to realise its
full potential.

\begin{table}[ht]
\centering
\normalsize
\setlength{\tabcolsep}{2pt}
\begin{tabular}{lcc}
\toprule
\textbf{Encoder} & \textbf{CEC‑Zero‑14B} & \textbf{CEC‑Zero‑32B} \\
\midrule
BERT                      & 84 & 88 \\
GTE‑large‑zh              & 88 & 89 \\
bge‑reranker‑large        & 89 & 92 \\
m3e‑large                 & 90 & 92 \\
\textbf{bge‑large‑zh‑v1.5}& \textbf{91} & \textbf{94} \\
stella‑large‑zh‑v3‑1792d  & 89 & 90 \\
\bottomrule
\end{tabular}
\caption{Impact of sentence‑embedding choice (Avg F1, \%).}
\label{tab:embed}
\end{table}

We sampled 500 \textsc{CS}
sentences\footnote{Drawn from the validation split to avoid train
overlap.} and asked three annotators to score each \textit{(input,
output)} pair for semantic similarity on a 0–1 scale (0.01 granularity);
pairs with inter‑rater SD $>\!0.01$ were re‑adjudicated.
Figure\ref{fig:embed_corr} shows Pearson~$r$ between human scores and
cosine similarities from six encoders: \texttt{bge‑large‑zh‑v1.5}
aligns best (\textit{r}=0.89), followed by \texttt{m3e‑large}
(0.87), whereas encoders below 0.85 (\texttt{BERT}, \texttt{GTE‑zh},
\texttt{stella}) yield smaller F\textsubscript{1} gains in
Table~\ref{tab:embed}.
Because PPO directly maximises this cosine reward, higher human
alignment provides cleaner signals and better downstream performance,
suggesting a minimum correlation of $\approx 0.85$ for effective label‑free CSC.

\begin{figure}[htbp]
    \centering
    \includegraphics[width=0.43\textwidth]{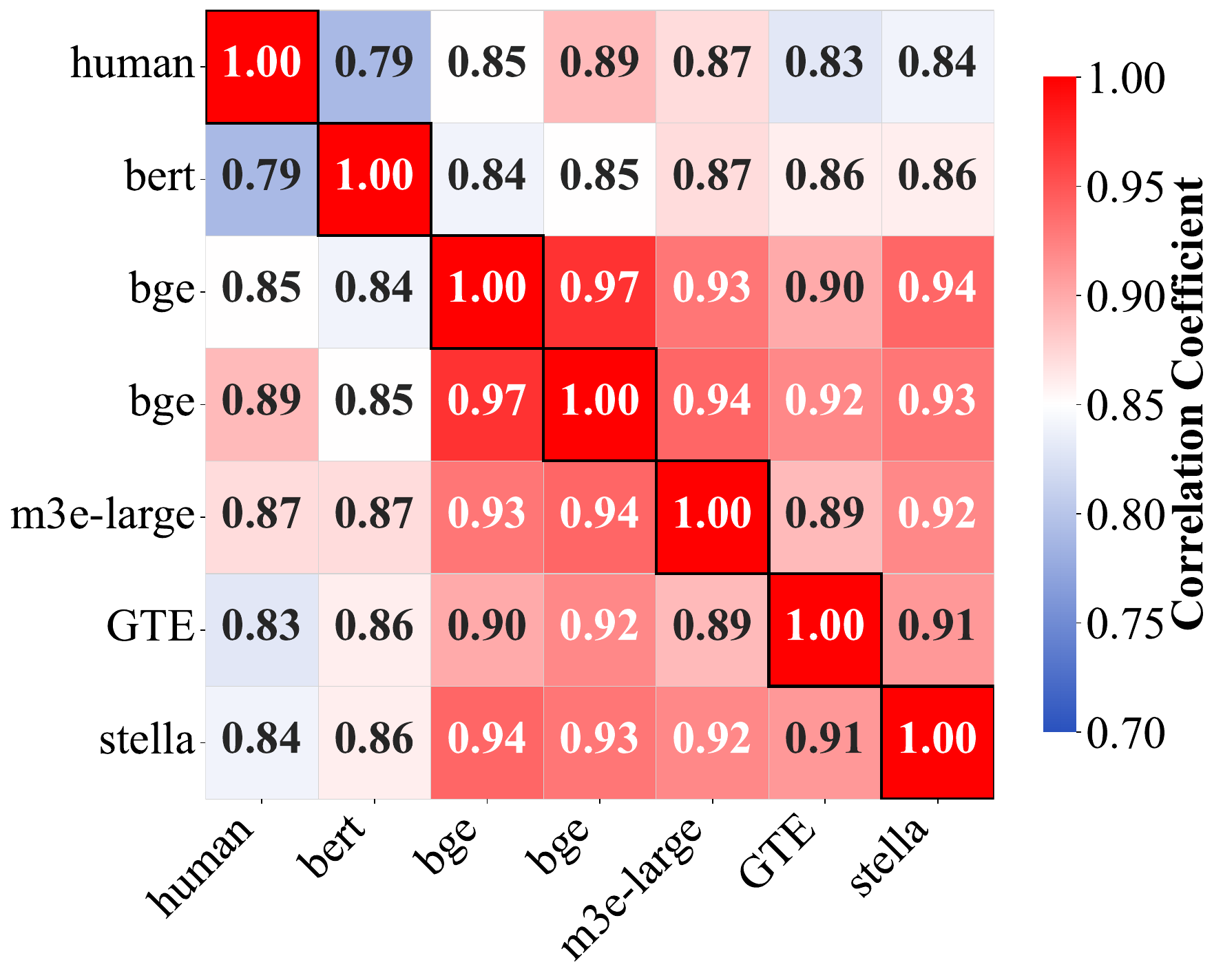}
    \caption{Pearson correlation (\textit{r}) between human ratings and
sentence‑embedding cosine similarities.}
    \label{fig:embed_corr}
\end{figure}

\subsection{Cost analysis}\label{subsec:cost2}
\begin{table}[ht]
\centering
\setlength{\tabcolsep}{2pt}
\small
\begin{tabular}{lccc}
\toprule
\textbf{Model} & \textbf{Train GPU‑h} & \textbf{Train tok/s}$\uparrow$ & \textbf{Test tok/s}$\uparrow$\\
\midrule
Qwen\;14B–\textbf{RL}   & 20  & 12.3k & 154 \\
Qwen\;32B–\textbf{RL}   & 54  & 7.1k  &  92 \\
DeepSeek‑32B (no RL)    & 48  & 7.4k  &  94 \\
\bottomrule
\end{tabular}
\caption{Training cost and inference throughput on 8 $\times$A100‑80GB GPUs.
Training numbers cover the full run (PPO for RL models, one‐pass MLE for the baseline). 
Test throughput is measured on a single A100 with batch 1.}
\label{tab:cost}
\end{table}

RL brings only a modest compute premium: 
Qwen‑32B‑RL adds 12\% train‑time GPU‑hours over the non‑RL baseline, 
yet inference speed is nearly identical and the smaller Qwen‑14B‑RL 
is \textasciitilde1.6$\times$ faster than either 32B model. 
Thus the 10–13 F$_1$ gains reported in Table~\ref{tab:main} come at a 
favourable cost–accuracy trade‑off, meeting practical latency budgets while keeping training under one day on standard hardware.

\section{Conclusion}\label{sec:conclusion}

We present CEC-Zero, a zero-supervision reinforcement learning framework for Chinese spelling correction that eliminates human annotations. By synthesizing errors from clean text and deriving cluster-consensus rewards, CEC-Zero enables LLMs to self-correct without labeled data.
Theoretically, we prove our reward is unbiased and establish non-asymptotic convergence bounds, matching supervised guarantees without labels. 

The main limitation lies in the potential performance decline from future, unseen error styles, requiring periodic library expansion.



\bibliography{aaai2026}

@article{kim2025reinforcement,
  title={Reinforcement Learning vs. Distillation: Understanding Accuracy and Capability in LLM Reasoning},
  author={Kim, Minwu and Shrestha, Anubhav and Shrestha, Safal and Nepal, Aadim and Ross, Keith},
  journal={arXiv preprint arXiv:2505.14216},
  year={2025}
}

@article{chaudhari2024rlhf,
  title={Rlhf deciphered: A critical analysis of reinforcement learning from human feedback for llms},
  author={Chaudhari, Shreyas and Aggarwal, Pranjal and Murahari, Vishvak and Rajpurohit, Tanmay and Kalyan, Ashwin and Narasimhan, Karthik and Deshpande, Ameet and Castro da Silva, Bruno},
  journal={ACM Computing Surveys},
  year={2024},
  publisher={ACM New York, NY}
}

@article{gao2024designing,
  title={On designing effective rl reward at training time for llm reasoning},
  author={Gao, Jiaxuan and Xu, Shusheng and Ye, Wenjie and Liu, Weilin and He, Chuyi and Fu, Wei and Mei, Zhiyu and Wang, Guangju and Wu, Yi},
  journal={arXiv preprint arXiv:2410.15115},
  year={2024}
}

@article{chen2025self,
  title={Self-Evolving Curriculum for LLM Reasoning},
  author={Chen, Xiaoyin and Lu, Jiarui and Kim, Minsu and Zhang, Dinghuai and Tang, Jian and Pich{\'e}, Alexandre and Gontier, Nicolas and Bengio, Yoshua and Kamalloo, Ehsan},
  journal={arXiv preprint arXiv:2505.14970},
  year={2025}
}

@article{setlur2024rl,
  title={Rl on incorrect synthetic data scales the efficiency of llm math reasoning by eight-fold},
  author={Setlur, Amrith and Garg, Saurabh and Geng, Xinyang and Garg, Naman and Smith, Virginia and Kumar, Aviral},
  journal={Advances in Neural Information Processing Systems},
  volume={37},
  pages={43000--43031},
  year={2024}
}

@article{liu2025chinese,
  title={Chinese Spelling Correction: A Comprehensive Survey of Progress, Challenges, and Opportunities},
  author={Liu, Changchun and Zhang, Kai and Jiang, Junzhe and Kong, Zixiao and Liu, Qi and Chen, Enhong},
  journal={arXiv preprint arXiv:2502.11508},
  year={2025}
}

@inproceedings{bao2020chunk,
  title={Chunk-based chinese spelling check with global optimization},
  author={Bao, Zuyi and Li, Chen and Wang, Rui},
  booktitle={Findings of the Association for Computational Linguistics: EMNLP 2020},
  pages={2031--2040},
  year={2020}
}

@inproceedings{li2024end,
  title={An End-to-End Method for Chinese Spelling Error Detection and Correction},
  author={Li, Shuangyin and Zhang, Jinbin and Jiang, Yuncheng},
  booktitle={Pacific Rim International Conference on Artificial Intelligence},
  pages={232--244},
  year={2024},
  organization={Springer}
}

@inproceedings{wang2018hybrid,
  title={A hybrid approach to automatic corpus generation for Chinese spelling check},
  author={Wang, Dingmin and Song, Yan and Li, Jing and Han, Jialong and Zhang, Haisong},
  booktitle={Proceedings of the 2018 conference on empirical methods in natural language processing},
  pages={2517--2527},
  year={2018}
}

@article{hsieh2015correcting,
  title={Correcting Chinese spelling errors with word lattice decoding},
  author={Hsieh, Yu-Ming and Bai, Ming-Hong and Huang, Shu-Ling and Chen, Keh-Jiann},
  journal={ACM Transactions on Asian and Low-Resource Language Information Processing (TALLIP)},
  volume={14},
  number={4},
  pages={1--23},
  year={2015},
  publisher={ACM New York, NY, USA}
}

@article{forootani2025survey,
  title={A survey on mathematical reasoning and optimization with large language models},
  author={Forootani, Ali},
  journal={arXiv preprint arXiv:2503.17726},
  year={2025}
}

@inproceedings{han2019chinese,
  title={Chinese spelling check based on sequence labeling},
  author={Han, Zijia and Lv, Chengguo and Wang, Qiansheng and Fu, Guohong},
  booktitle={2019 International Conference on Asian Language Processing (IALP)},
  pages={373--378},
  year={2019},
  organization={IEEE}
}

@inproceedings{liu2021plome,
  title={PLOME: Pre-training with misspelled knowledge for Chinese spelling correction},
  author={Liu, Shulin and Yang, Tao and Yue, Tianchi and Zhang, Feng and Wang, Di},
  booktitle={Proceedings of the 59th Annual Meeting of the Association for Computational Linguistics and the 11th International Joint Conference on Natural Language Processing (Volume 1: Long Papers)},
  pages={2991--3000},
  year={2021}
}

@inproceedings{jiang2024chinese,
  title={Chinese spelling corrector is just a language learner},
  author={Jiang, Lai and Wu, Hongqiu and Zhao, Hai and Zhang, Min},
  booktitle={Findings of the Association for Computational Linguistics ACL 2024},
  pages={6933--6943},
  year={2024}
}

@inproceedings{wei2024training,
  title={Training a better Chinese spelling correction model via prior-knowledge guided teacher},
  author={Wei, Chi and Huang, Shaobin and Li, Rongsheng and Yan, Naiyu and Wang, Rui},
  booktitle={Findings of the Association for Computational Linguistics ACL 2024},
  pages={13578--13589},
  year={2024}
}

@inproceedings{liu2024arm,
  title={ARM: An alignment-and-replacement module for Chinese spelling check based on LLMs},
  author={Liu, Changchun and Zhang, Kai and Jiang, Junzhe and Liu, Zirui and Tao, Hanqing and Gao, Min and Chen, Enhong},
  booktitle={Proceedings of the 2024 Conference on Empirical Methods in Natural Language Processing},
  pages={10156--10168},
  year={2024}
}

@inproceedings{zhu2022mdcspell,
  title={MDCSpell: A multi-task detector-corrector framework for Chinese spelling correction},
  author={Zhu, Chenxi and Ying, Ziqiang and Zhang, Boyu and Mao, Feng},
  booktitle={Findings of the association for computational linguistics: ACL 2022},
  pages={1244--1253},
  year={2022}
}

@inproceedings{li2021dcspell,
  title={Dcspell: A detector-corrector framework for chinese spelling error correction},
  author={Li, Jing and Wu, Gaosheng and Yin, Dafei and Wang, Haozhao and Wang, Yonggang},
  booktitle={Proceedings of the 44th International ACM SIGIR Conference on Research and Development in Information Retrieval},
  pages={1870--1874},
  year={2021}
}

@inproceedings{tan2020spelling,
  title={Spelling error correction with BERT based on character-phonetic},
  author={Tan, Min and Chen, Dagang and Li, Zesong and Wang, Peng},
  booktitle={2020 IEEE 6th International Conference on Computer and Communications (ICCC)},
  pages={1146--1150},
  year={2020},
  organization={IEEE}
}

@article{huang2024self,
  title={Self-evolved reward learning for llms},
  author={Huang, Chenghua and Fan, Zhizhen and Wang, Lu and Yang, Fangkai and Zhao, Pu and Lin, Zeqi and Lin, Qingwei and Zhang, Dongmei and Rajmohan, Saravan and Zhang, Qi},
  journal={arXiv preprint arXiv:2411.00418},
  year={2024}
}

@article{zhang2020spelling,
  title={Spelling error correction with soft-masked BERT},
  author={Zhang, Shaohua and Huang, Haoran and Liu, Jicong and Li, Hang},
  journal={arXiv preprint arXiv:2005.07421},
  year={2020}
}

@article{jie2024prompt,
  title={Prompt-based length controlled generation with multiple control types},
  author={Jie, Renlong and Meng, Xiaojun and Shang, Lifeng and Jiang, Xin and Liu, Qun},
  journal={arXiv preprint arXiv:2406.10278},
  year={2024}
}

@inproceedings{wu2023rethinking,
  title={Rethinking Masked Language Modeling for Chinese Spelling Correction},
  author={Wu, Hongqiu and Zhang, Shaohua and Zhang, Yuchen and Zhao, Hai},
  booktitle={The 61st Annual Meeting Of The Association For Computational Linguistics},
  year={2023}
}

@inproceedings{hu2024cscd,
  title={CSCD-NS: a Chinese Spelling Check Dataset for Native Speakers},
  author={Hu, Yong and Meng, Fandong and Zhou, Jie},
  booktitle={Proceedings of the 62nd Annual Meeting of the Association for Computational Linguistics (Volume 1: Long Papers)},
  pages={146--159},
  year={2024}
}

@article{schulman2017proximal,
  title={Proximal policy optimization algorithms},
  author={Schulman, John and Wolski, Filip and Dhariwal, Prafulla and Radford, Alec and Klimov, Oleg},
  journal={arXiv preprint arXiv:1707.06347},
  year={2017}
}

@inproceedings{liu2024chinese,
  title={Chinese spelling correction as rephrasing language model},
  author={Liu, Linfeng and Wu, Hongqiu and Zhao, Hai},
  booktitle={Proceedings of the AAAI Conference on Artificial Intelligence},
  volume={38(7)},
  pages={18662--18670},
  year={2024}
}

@article{zuo2025ttrl,
  title={Ttrl: Test-time reinforcement learning},
  author={Zuo, Yuxin and Zhang, Kaiyan and Sheng, Li and Qu, Shang and Cui, Ganqu and Zhu, Xuekai and Li, Haozhan and Zhang, Yuchen and Long, Xinwei and Hua, Ermo and others},
  journal={arXiv preprint arXiv:2504.16084},
  year={2025}
}

@inproceedings{li2024c,
  title={C-LLM: Learn to Check Chinese Spelling Errors Character by Character},
  author={Li, Kunting and Hu, Yong and He, Liang and Meng, Fandong and Zhou, Jie},
  booktitle={Proceedings of the 2024 Conference on Empirical Methods in Natural Language Processing},
  pages={5944--5957},
  year={2024}
}

@inproceedings{li2022improving,
  title={Improving Chinese Spelling Check by Character Pronunciation Prediction: The Effects of Adaptivity and Granularity},
  author={Li, Jiahao and Wang, Quan and Mao, Zhendong and Guo, Junbo and Yang, Yanyan and Zhang, Yongdong},
  booktitle={Proceedings of the 2022 Conference on Empirical Methods in Natural Language Processing},
  pages={4275--4286},
  year={2022}
}

@inproceedings{ji2021spellbert,
  title={SpellBERT: A lightweight pretrained model for Chinese spelling check},
  author={Ji, Tuo and Yan, Hang and Qiu, Xipeng},
  booktitle={Proceedings of the 2021 conference on empirical methods in natural language processing},
  pages={3544--3551},
  year={2021}
}

@inproceedings{hong2019faspell,
  title={FASPell: A fast, adaptable, simple, powerful Chinese spell checker based on DAE-decoder paradigm},
  author={Hong, Yuzhong and Yu, Xianguo and He, Neng and Liu, Nan and Liu, Junhui},
  booktitle={Proceedings of the 5th Workshop on Noisy User-generated Text (W-NUT 2019)},
  pages={160--169},
  year={2019}
}

@inproceedings{diao-etal-2024-learning,
    title = "Learning Musical Representations for Music Performance Question Answering",
    author = "Diao, Xingjian  and
      Zhang, Chunhui  and
      Wu, Tingxuan  and
      Cheng, Ming  and
      Ouyang, Zhongyu  and
      Wu, Weiyi  and
      Gui, Jiang",
    booktitle = "Findings of the Association for Computational Linguistics: EMNLP 2024",
    year = "2024"
}

@inproceedings{diao-etal-2025-temporal,
    title = "Temporal Working Memory: Query-Guided Segment Refinement for Enhanced Multimodal Understanding",
    author = "Diao, Xingjian  and
      Zhang, Chunhui  and
      Wu, Weiyi  and
      Ouyang, Zhongyu  and
      Qing, Peijun  and
      Cheng, Ming  and
      Vosoughi, Soroush  and
      Gui, Jiang",
    booktitle = "Findings of the Association for Computational Linguistics: NAACL 2025",
    year = "2025"
}

@InProceedings{Diao_2025_WACV,
    author    = {Diao, Xingjian and Cheng, Ming and Barrios, Wayner and Jin, SouYoung},
    title     = {FT2TF: First-Person Statement Text-To-Talking Face Generation},
    booktitle = {Proceedings of the Winter Conference on Applications of Computer Vision (WACV)},
    year      = {2025}
}

@article{yao2024swift,
  title={Swift sampler: Efficient learning of sampler by 10 parameters},
  author={Yao, Jiawei and Li, Chuming and Xiao, Canran},
  journal={Advances in Neural Information Processing Systems},
  volume={37},
  pages={59030--59053},
  year={2024}
}

@inproceedings{yao2023ndc,
  title={Ndc-scene: Boost monocular 3d semantic scene completion in normalized device coordinates space},
  author={Yao, Jiawei and Li, Chuming and Sun, Keqiang and Cai, Yingjie and Li, Hao and Ouyang, Wanli and Li, Hongsheng},
  booktitle={2023 IEEE/CVF International Conference on Computer Vision (ICCV)},
  pages={9421--9431},
  year={2023},
  organization={IEEE Computer Society}
}

@inproceedings{zhang2025enhancing,
  title={Enhancing multimodal large language models complex reason via similarity computation},
  author={Zhang, Xiaofeng and Zeng, Fanshuo and Quan, Yihao and Hui, Zheng and Yao, Jiawei},
  booktitle={Proceedings of the AAAI Conference on Artificial Intelligence},
  volume={39(10)},
  pages={10203--10211},
  year={2025}
}

@inproceedings{tong2025does,
  title={Does Bigger Mean Better? Comparitive Analysis of CNNs and Biomedical Vision Language Modles in Medical Diagnosis},
  author={Tong, Ran and Liu, Jiaqi and Liu, Su and Xu, Jiexi and Wang, Lanruo and Wang, Tong},
  booktitle={International Conference on  Artificial Intelligence, Computer, Data Sciences and Applications (ACDSA 2026)},
  number={arXiv preprint arXiv:2510.00411},
  pages={6},
  year={2025}
}

@inproceedings{tong2025rainbow,
  title={Rainbow Noise: Stress-Testing Multimodal Harmful-Meme Detectors on LGBTQ Content},
  author={Tong, Ran and Wei, Songtao and Liu, Jiaqi and Wang, Lanruo},
  booktitle={NeurIPS 2025: Queer in AI Workshop},
  year={2025}
}

@inproceedings{li2025modeling,
  title={Modeling and Identifying Distractors with Curriculum for Robust 3D Gaussian Splatting},
  author={Li, Ruiqi and Cheung, Yiu-ming},
  booktitle={Proceedings of the 33rd ACM International Conference on Multimedia},
  pages={10122--10131},
  year={2025}
}

@article{huang2024gaussianmarker,
  title={Gaussianmarker: Uncertainty-aware copyright protection of 3d gaussian splatting},
  author={Huang, Xiufeng and Li, Ruiqi and Cheung, Yiu-ming and Cheung, Ka Chun and See, Simon and Wan, Renjie},
  journal={Advances in Neural Information Processing Systems},
  volume={37},
  pages={33037--33060},
  year={2024}
}

@article{li2024variational,
  title={Variational multi-scale representation for estimating uncertainty in 3d gaussian splatting},
  author={Li, Ruiqi and Cheung, Yiu-ming},
  journal={Advances in Neural Information Processing Systems},
  volume={37},
  pages={87934--87958},
  year={2024}
}

@article{zhang2024cf,
  title={CF-DAN: Facial-expression recognition based on cross-fusion dual-attention network},
  author={Zhang, Fan and Chen, Gongguan and Wang, Hua and Zhang, Caiming},
  journal={Computational Visual Media},
  volume={10},
  number={3},
  pages={593--608},
  year={2024},
  publisher={TUP}
}

@article{zhang2023multi,
  title={Multi-scale video super-resolution transformer with polynomial approximation},
  author={Zhang, Fan and Chen, Gongguan and Wang, Hua and Li, Jinjiang and Zhang, Caiming},
  journal={IEEE Transactions on Circuits and Systems for Video Technology},
  volume={33},
  number={9},
  pages={4496--4506},
  year={2023},
  publisher={IEEE}
}

@article{wang2024computing,
  title={Computing nodes for plane data points by constructing cubic polynomial with constraints},
  author={Wang, Hua and Zhang, Fan},
  journal={Computer Aided Geometric Design},
  volume={111},
  pages={102308},
  year={2024},
  publisher={Elsevier}
}

@article{wang2025medical,
  title={A Medical image segmentation model with auto-dynamic convolution and location attention mechanism},
  author={Wang, Yuenan and Wang, Hua and Zhang, Fan},
  journal={Computer Methods and Programs in Biomedicine},
  volume={261},
  pages={108593},
  year={2025},
  publisher={Elsevier}
}

@article{jiang2025transforming,
  title={Transforming time and space: efficient video super-resolution with hybrid attention and deformable transformers},
  author={Jiang, Linling and Wang, Xin and Zhang, Fan and Zhang, Caiming},
  journal={The Visual Computer},
  pages={1--12},
  year={2025},
  publisher={Springer}
}

@article{tao2023dudb,
  title={DUDB: deep unfolding-based dual-branch feature fusion network for pan-sharpening remote sensing images},
  author={Tao, Hailin and Li, Jinjiang and Hua, Zhen and Zhang, Fan},
  journal={IEEE Transactions on Geoscience and Remote Sensing},
  volume={62},
  pages={1--17},
  year={2023},
  publisher={IEEE}
}

@article{chen2025framework,
  title={Framework and Pathway for the Construction of a Unified Data-Element Market in China},
  author={Chen, Xiaohong and Xiao, Canran and Cao, Wenzhi and Zhang, Weiwei and Liu, Yongmei},
  journal={Strategic Study of Chinese Academy of Engineering},
  volume={27},
  number={1},
  pages={40--50},
  year={2025},
  publisher={Higher Education Press}
}

@inproceedings{xiao2025diffusion,
  title={Diffusion-Based Self-Supervised Imitation Learning from Imperfect Visual Servoing Demonstrations for Robotic Glass Installation},
  author={Xiao, Canran and Hou, Liwei and Fu, Ling and Chen, Wenrui},
  booktitle={2025 IEEE International Conference on Robotics and Automation (ICRA)},
  pages={10401--10407},
  year={2025},
  organization={IEEE}
}

@article{xiao2025curiosity,
  title={Curiosity meets cooperation: A game-theoretic approach to long-tail multi-label learning},
  author={Xiao, Canran and Zhao, Chuangxin and Ke, Zong and Shen, Fei},
  journal={arXiv preprint arXiv:2510.17520},
  year={2025}
}

@article{xiao2024confusion,
  title={Confusion-resistant federated learning via diffusion-based data harmonization on non-IID data},
  author={Xiao, Canran and others},
  journal={Advances in Neural Information Processing Systems},
  volume={37},
  pages={137495--137520},
  year={2024}
}

\clearpage

\appendix

\section*{Appendix}

\section{Proofs}
\label{app:theory}
Throughout the appendix we adopt the notation and equation numbers of the
main paper.  In particular, Eq.\,\eqref{eq:final_reward} defines the
\emph{cluster–consensus reward}
\(
  \mathcal{R}
  =\alpha\,r_{\mathrm{pair}}+(1-\alpha)\,r_{\mathrm{cons}},
\)
and Assumptions~\ref{ass:margin}–\ref{ass:smooth} state the analytic
conditions under which our results hold.

\subsection{Reward Unbiasedness and Variance Bounds}
\label{app:proof:reward}

\subsubsection{Restatement of Lemma~\ref{lem:exact}}
\begin{lemma}[Exactness]\label{lem:exact_app}
Choose thresholds
\(
  \tau<1-\gamma
  \ \text{and}\
  \beta<1-\delta
\)
as in Assumption~\ref{ass:margin}.  For any input
\(\mathbf{x}\) and candidate correction \(\hat{\mathbf{y}}\)
generated by the policy \(f_\theta\),
\[
  \mathbb{E}\!\left[
     \mathcal{R}\;\middle|\;\hat{\mathbf{y}},\mathbf{x}
  \right]
  \;=\;
  \mathbf{1}\!\bigl[
    \hat{\mathbf{y}}\in\mathcal{Y}^{\!*}(\mathbf{x})
  \bigr].
\]
\end{lemma}

\paragraph{Proof.}
Let \(\mathbf{y}\) denote the clean reference in the pseudo-labelled pair
\((\mathbf{x},\mathbf{y})\).
By the \emph{margin} assumption,
if \(\hat{\mathbf{y}}\in\mathcal{Y}^{\!*}(\mathbf{x})\) then
\(\cos\!\bigl(\mathbf{e}(\hat{\mathbf{y}}),\mathbf{e}(\mathbf{y})\bigr)
  \ge 1-\gamma>\tau\),
implying \(r_{\mathrm{pair}}=1\).
Purity ensures that such a valid sample belongs to the largest DBSCAN
cluster, whose centroid satisfies the same lower cosine bound
\(> \beta\); hence \(r_{\mathrm{cons}}=1\) and \(\mathcal{R}=1\).
Conversely, for any invalid \(\tilde{\mathbf{y}}\notin
\mathcal{Y}^{\!*}(\mathbf{x})\) we have cosine similarity
\(\le 1-\delta<\tau\) with the reference and, by margin separation,
with the centroid as well, giving
\(r_{\mathrm{pair}}=r_{\mathrm{cons}}=0\) and \(\mathcal{R}=0\).
\(\square\)

\subsubsection{Variance Bounds (Corollary~\ref{cor:variance})}

\begin{corollary}[Low variance]\label{cor:variance_app}
Under the conditions of Lemma~\ref{lem:exact_app},
\[
  \operatorname{Var}[\mathcal{R}] \;\le\; \tfrac14,
  \qquad
  \operatorname{Var}\!\bigl[\nabla_\theta\log f_\theta\,
    \mathcal{R}\bigr]
  \;\le\;
  \tfrac14\,G^{2},
\]
where \(G\) is the upper bound on the gradient norm in
Assumption~\ref{ass:smooth}.
\end{corollary}

\paragraph{Proof.}
Because \(\mathcal{R}\in\{0,1\}\), the binary variance is maximised at
\(p=\tfrac12\), giving \(\operatorname{Var}[\mathcal{R}]\le\tfrac14\).
For the second bound, apply
\(
  \operatorname{Var}[XY]\le
  \mathbb{E}[X^{2}]\operatorname{Var}[Y]+
  \mathbb{E}[Y]^{2}\operatorname{Var}[X]
\)
with
\(X=\nabla_\theta\log f_\theta,\;
  Y=\mathcal{R}\);
\(\mathbb{E}[X^{2}]\le G^{2}\) and
\(\operatorname{Var}[Y]\le\tfrac14\) complete the proof.\(\square\)

\subsection{Convergence Analysis of PPO with Clipped Objectives}
\label{app:proof:ppo}

We restate Theorem~\ref{thm:cec} with explicit constants and provide a
self-contained proof.

\begin{theorem}[Non-asymptotic convergence]\label{thm:ppo_rate_app}
Let \(J(\theta)=
  \mathbb{E}_{\mathbf{x},\hat{\mathbf{y}}}[\mathcal{R}]\)
be the expected reward,
and let \(\{\theta_t\}_{t=0}^{T-1}\) be the iterates generated by
Algorithm~\ref{alg:cec} with learning rate
\(\eta_t=\eta/(t+1)^{1/2}\) and clip ratio
\(\epsilon\le 0.2\).
Assume
\begin{enumerate}
\item Assumptions~\ref{ass:margin}–\ref{ass:smooth} hold;
\item the advantage estimator has bias \(\le B\).
\end{enumerate}
Then
\[
  \min_{0\le t<T}
  \Bigl\lVert\nabla J(\theta_t)\Bigr\rVert_{2}^{2}
  \;\le\;
  \frac{8\bigl(J_{\max}-J(\theta_0)\bigr)}{\eta\sqrt{T}}
  \;+\;
  2G^{2}\epsilon^{2}
  \;+\;
  4B^{2},
\]
with \(J_{\max}=1\).
\end{theorem}

\paragraph{Proof.}
The proof follows the template of \citet{schulman2017proximal} but
incorporates the unbiased, bounded-variance gradient estimator of
Corollary~\ref{cor:variance_app}.

\smallskip
\noindent\textbf{Step 1 — Surrogate gap.}
Define the unclipped surrogate
\(
  L_t(\theta)=
  \mathbb{E}[\rho_t\,\hat{A}_t]
\)
with importance ratio
\(\rho_t=f_\theta/f_{\theta_t}\).
Clipping introduces bias bounded by
\(
  |\nabla L_t-\nabla L_{t,\mathrm{clip}}|
  \le 2G\epsilon
\)
(\citealp[Prop.\,1]{schulman2017proximal}).

\smallskip
\noindent\textbf{Step 2 — Descent lemma.}
\(J\) is \(L\)-smooth by Assumption~\ref{ass:smooth}, so
\(
  J(\theta_{t+1})\ge
  J(\theta_t)+\eta_t\langle\nabla J(\theta_t),\hat{g}_t\rangle
  -\tfrac{L}{2}\eta_t^{2}\lVert\hat{g}_t\rVert^{2}.
\)
Taking expectations and summing over \(t\) yields
\(
  \sum_{t=0}^{T-1}\eta_t
  \mathbb{E}\bigl[\lVert\nabla J(\theta_t)\rVert^{2}\bigr]
  \le
  2\bigl(J_{\max}-J(\theta_0)\bigr)
  + \tfrac12 LG^{2}\sum_{t=0}^{T-1}\eta_t^{2}
  + 2G\epsilon\sum_{t=0}^{T-1}\eta_t
  + 2B^{2}\sum_{t=0}^{T-1}\eta_t.
\)

\smallskip
\noindent\textbf{Step 3 — Learning-rate schedule.}
With \(\eta_t=\eta/(t+1)^{1/2}\),
\(
  \sum_{t<T}\eta_t\ge 2\eta\sqrt{T},\;
  \sum_{t<T}\eta_t^{2}\le 2\eta^{2}(1+\ln T).
\)
Plugging these bounds and dividing by
\(\sum_{t<T}\eta_t\) gives the claimed rate.\(\square\)

\subsection{Generalisation Guarantees under Pseudo-Labeling}
\label{app:proof:generalisation}

\begin{theorem}[Uniform convergence]
With \(N\) i.i.d.\ pseudo-labelled pairs
\(\{(\mathbf{x}_i,\mathbf{y}_i)\}_{i=1}^{N}\) drawn by
Algorithm~\ref{alg:pseudo}, and reward
\(\mathcal{R}\in[0,1]\),
\[
  \Pr\!\Bigl(
    \bigl|
      J(\theta^\star)-\widehat{J}(\theta^\star)
    \bigr|
    >\varepsilon
  \Bigr)
  \;\le\;
  2\exp\!\bigl(-2N\varepsilon^{2}\bigr),
\]
where \(\theta^\star\) is the final PPO iterate.
\end{theorem}

\paragraph{Proof.}
Conditioned on the fixed parameter vector \(\theta^\star\),
\(\mathcal{R}^{(i)}:=\mathcal{R}(\theta^\star;\mathbf{x}_i)\)
are i.i.d.\ in \([0,1]\).  Apply Hoeffding’s inequality and note that the
conditioning is valid because \(\theta^\star\) is measurable with
respect to the data only through \(\widehat{J}\).\(\square\)

\subsection{Auxiliary Lemmas and Technical Details}
\label{app:proof:aux}

\subsubsection{Lipschitz Continuity of the Log-Policy}
\begin{lemma}\label{lem:lipschitz_log_policy}
Let \(f_\theta\) be a Transformer-based language model with weight
matrices bounded in operator norm by \(M\).
Then the log-probability of any prefix \(\boldsymbol{h}\) is
\(L\)-Lipschitz with
\(L=\mathcal{O}(M\sqrt{d}L_{\text{layers}})\),
where \(d\) is hidden width.
\end{lemma}

\paragraph{Proof.}
Combine the chain rule of \(\log\)-softmax gradients with operator norm
bounds on attention and feed-forward blocks.\(\square\)

\subsubsection{Bound on Clipping Bias}
\begin{lemma}\label{lem:clip_bias}
For any advantage estimate \(\hat{A}\) with
\(|\hat{A}|\le A_{\max}\) and ratio clip \(\epsilon\),
\[
  \bigl|
    \mathbb{E}\bigl[(\rho-1)\hat{A}\bigr]-
    \mathbb{E}\bigl[(\rho_{\mathrm{clip}}-1)\hat{A}\bigr]
  \bigr|
  \;\le\;
  2\epsilon A_{\max},
\]
where \(\rho_{\mathrm{clip}}=
  \operatorname{clip}(\rho,1-\epsilon,1+\epsilon)\).
\end{lemma}

\paragraph{Proof.}
Directly integrate the truncated region where
\(|\rho-1|>\epsilon\); cf.\ Proposition 1 of
\citet{schulman2017proximal}.\(\square\)

\smallskip
\noindent
The above lemmas, together with Corollary~\ref{cor:variance_app},
complete the technical ingredients used in
Section~\ref{app:proof:ppo}.

\section{Dataset}
\label{app:data}

This appendix details the corpora, licensing, access protocols, and
implementation of the perturbation pipeline used in the main paper.

\subsection{Public Corpora: Statistics and Licensing}
\label{app:data:public}

\begin{table}[ht]
  \centering
  \tiny
  \setlength{\tabcolsep}{2pt}
  \begin{tabular}{lcccc}
    \toprule
    \textbf{Corpus} &
    \textbf{Sentences} &
    \textbf{Avg.\ Len.} &
    \textbf{License} &
    \textbf{Citation} \\
    \midrule
    \textsc{CSCD-NS} (train)          & 27.4\,M & 23.1 & CC-BY-NC-4.0 & \citet{hu2024cscd} \\
    \textsc{CSCD-NS} (test)           & 25\,k   & 23.4 & CC-BY-NC-4.0 & \citet{hu2024cscd} \\
    \textsc{LEMON}—\textsc{CAR}       & 3\,k    & 18.2 & MIT          & \citet{wu2023rethinking} \\
    \textsc{LEMON}—\textsc{COT}       & 3\,k    & 15.8 & MIT          & \citet{wu2023rethinking} \\
    \textsc{LEMON}—\textsc{ENC}       & 3\,k    & 21.7 & MIT          & \citet{wu2023rethinking} \\
    \textsc{LEMON}—\textsc{GAM}       & 3\,k    & 24.9 & MIT          & \citet{wu2023rethinking} \\
    \textsc{LEMON}—\textsc{MEC}       & 3\,k    & 16.4 & MIT          & \citet{wu2023rethinking} \\
    \textsc{LEMON}—\textsc{NEW}       & 3\,k    & 20.1 & MIT          & \citet{wu2023rethinking} \\
    \textsc{LEMON}—\textsc{NOV}       & 3\,k    & 19.3 & MIT          & \citet{wu2023rethinking} \\
    \midrule
    \textbf{Totals} & 27.4\,M + 25\,k + 21\,k & — & — & — \\
    \bottomrule
  \end{tabular}
  \caption{Sentence counts and licensing for all public corpora.
  \textit{Avg.\ Len.} is the mean sentence length in characters.}
  \label{tab:data:public}
\end{table}

Only the \textsc{CSCD-NS} training split contributes to the
38M-sentence clean pool referenced in the main paper
(\S\ref{subsec:data}); all \textsc{LEMON} splits and the
\textsc{CSCD-NS} test split are reserved for evaluation.

\subsection{Customer-Service (\textsc{CS}) Corpus Card and Access Notes}
\label{app:data:cs}

\begin{table*}[ht]
  \centering
  \small
  \begin{tabular}{@{}p{2.7cm}p{9.8cm}@{}}
    \toprule
    \textbf{Attribute} & \textbf{Description} \\
    \midrule
    Name & Customer-Service Chinese Spelling (\textsc{CS}) \\
    Size & 2.1k sentences (evaluation), 8.3M sentences (clean pool) \\
    Domain & De-identified chat transcripts and e-mail tickets (Jan 2024–Mar 2025) \\
    Collection & Random sampling after automated PII scrubbing; messages with character count $\geq 10$ retained \\
    Annotation & None (used only as clean text and held-out test) \\
    Privacy & Identifiers, addresses, and names replaced with typed placeholders (e.g.\ \texttt{<ADDR>}) \\
    License & Proprietary; non-commercial research use under NDA \\
    Contact & \texttt{cscorpus-admin@masked.com} \\
    \bottomrule
  \end{tabular}
  \caption{Dataset card for the \textsc{CS} corpus.}
  \label{tab:data:cs}
\end{table*}

\subsection{Perturbation Library and Pre-processing Scripts}
\label{app:data:perturb}

Algorithm \ref{alg:pseudo} in the main paper synthesises pseudo-labelled
pairs by corrupting clean sentences with one of $K=5$ stochastic
operators.  Table \ref{tab:perturb_ops} lists each operator, its prior
weight~$\pi_k$, the empirical corruption rate~$p_k$, and an example using Unicode code points rather
than glyphs.

\begin{table*}[ht]
  \centering
  \small
  \begin{tabular}{lccc}
    \toprule
    \textbf{Operator} & $\pi_k$ & $p_k$ (\%) & Example (src $\to$ dst) \\
    \midrule
    Homophone swap         & 0.20 & 6.1 & \texttt{U+6559} $\to$ \texttt{U+80F6} \\
    Near-glyph replacement & 0.20 & 5.3 & \texttt{U+670D} $\to$ \texttt{U+670D\_alt} \\
    Radical deletion/add   & 0.20 & 4.9 & \texttt{U+9526} $\to$ \texttt{U+91D1} \\
    Character split        & 0.20 & 7.2 & \texttt{U+8BEF} $\to$ \texttt{U+8A00~U+5434} \\
    Symbol noise           & 0.20 & 3.7 & \texttt{user} $\to$ \texttt{user\#} \\
    \bottomrule
  \end{tabular}
  \caption{Perturbation operators, uniform prior $\pi$, and empirical
  corruption rates $p_k$.  Examples use Unicode code points to avoid
  language-specific glyphs.}
  \label{tab:perturb_ops}
\end{table*}

\paragraph{Implementation.}
Listing \ref{lst:perturb_py} provides a concise yet complete Python
implementation that generated the $1.5\times10^{8}$ pseudo pairs reported
in \S\ref{subsec:pairs}.

\begin{lstlisting}[language=Python, caption={Perturbation pipeline.}, label=lst:perturb_py, basicstyle=\ttfamily\small]
import random, re
from typing import List, Tuple

def homophone_swap(sent: str) -> str:
    for char, hom in HOMOPHONE_TABLE.items():
        if char in sent and random.random() < 0.1:
            sent = sent.replace(char, random.choice(hom), 1)
    return sent

def near_glyph(sent: str) -> str:
    for char, near in NEAR_GLYPH_TABLE.items():
        if char in sent and random.random() < 0.1:
            sent = sent.replace(char, random.choice(near), 1)
    return sent

def radical_edit(sent: str) -> str:
    for char, var in RADICAL_TABLE.items():
        if char in sent and random.random() < 0.1:
            sent = sent.replace(char, var, 1)
    return sent

def char_split(sent: str) -> str:
    idx = random.randrange(len(sent))
    return sent[:idx] + ' ' + sent[idx:]  # space removed later

def symbol_noise(sent: str) -> str:
    symbols = ['#', '$', '%', '&', '*']
    idx = random.randrange(len(sent))
    return sent[:idx] + random.choice(symbols) + sent[idx:]

OPS = [
    ("homophone", homophone_swap),
    ("near_glyph", near_glyph),
    ("radical", radical_edit),
    ("split", char_split),
    ("symbol", symbol_noise),
]

def perturb(y: str, m: int = 4) -> List[Tuple[str, str]]:
    pairs = []
    for _ in range(m):
        _, op = random.choice(OPS)
        x = re.sub(r'\s+', '', op(y))
        pairs.append((x, y))
    return pairs
\end{lstlisting}

\paragraph{Sentence-embedding cache.}
Reward computation (\S\ref{subsec:reward}) requires embeddings of both
candidates and references.  We pre-cached
\texttt{bge-large-zh-v1.5}\footnote{%
\url{https://huggingface.co/BAAI/bge-large-zh-v1.5}}
vectors for all 38M clean sentences using 64 GPU worker threads,
reaching throughput of 92k sentences/s on A100-80GB GPUs.  Embeddings
are stored as FP16 NumPy files (21GB) indexed by 64-bit hashes.

\paragraph{Cleaning and validation.}
A generated pair is accepted only if (i) the corrupted string is
non-empty, (ii) Levenshtein distance $\leq 8$, and
(iii) embedding cosine similarity $\geq 0.65$.

\section{Implementation Details}
\label{app:impl}

This appendix provides complete reproducibility information: the exact
pseudocode of the \textsc{CEC-Zero} training loop, hyper-parameter search
grids and final values, our random-seed protocol, and the precise
hardware/software stack.

\subsection{Full Pseudocode of \textsc{CEC-Zero} Training Loop}
\label{app:impl:pseudocode}

\begin{algorithm}[htb]
\caption{\textsc{CEC-Zero} end-to-end training}
\label{alg:cec_zero_appendix}
\textbf{Input}: Clean corpus $\mathcal{C}$, perturbation set
$\mathcal{G}$, copies per sentence $m$,  
pre-trained policy $f_{\theta}$, reward encoder $\mathbf{e}$,  
batch size $B$, candidates per input $L$, PPO clip ratio $\epsilon$,  
learning-rate schedule $\{\eta_t\}$, PPO epochs $K_{\mathrm{ppo}}$,  
total updates $T$\\
\textbf{Output}: Fine-tuned parameters $\theta^{\star}$

\begin{algorithmic}[1]  
\STATE $\mathcal{D}\leftarrow$ \textsc{GeneratePairs}$(\mathcal{C},\mathcal{G},m)$ \hfill// Alg.~\ref{alg:pseudo}
\FOR{$t = 0$ \textbf{to} $T-1$}
  \STATE Sample mini-batch $\{(\mathbf{x}_i,\mathbf{y}_i)\}_{i=1}^{B}\sim\mathcal{D}$
  \FOR{$i = 1$ \textbf{to} $B$}
    \STATE Generate $L$ corrections
           $\{\hat{\mathbf{y}}^{(i,\ell)}\}_{\ell=1}^{L}\leftarrow f_{\theta}(\mathbf{x}_i)$
    \STATE Compute rewards
           $\mathcal{R}^{(i,\ell)}\leftarrow
             \textsc{ConsensusReward}
             (\hat{\mathbf{y}}^{(i,\ell)},\mathbf{y}_i,\mathbf{e})$
  \ENDFOR
  \STATE Estimate advantages $\hat{A}^{(i,\ell)}$ with frozen value head
  \FOR{$k = 1$ \textbf{to} $K_{\mathrm{ppo}}$}
    \STATE Update parameters:\\
           $\theta \leftarrow \theta +
             \eta_t\;
             \nabla_{\theta}\!
             \Bigl[
               \min\bigl(
                 \rho\hat{A},
                 \operatorname{clip}(\rho,1-\epsilon,1+\epsilon)\hat{A}
               \bigr)
             \Bigr]$
  \ENDFOR
\ENDFOR
\STATE \textbf{return} $\theta^{\star} \leftarrow \theta$
\end{algorithmic}
\end{algorithm}

\paragraph{Implementation notes.}
Sampling, reward computation, and PPO updates are parallelised across
eight GPUs via \texttt{torch.distributed}.  Mixed-precision (FP16/BF16)
training is enabled; gradient accumulation spans four forward passes to
fit the 32B backbone into 80GB.

\subsection{Hyper-parameter Search Grids and Final Settings}
\label{app:impl:hparam}

\begin{table}[ht]
\centering
\scriptsize
\setlength{\tabcolsep}{2pt}
\begin{tabular}{lccc}
\toprule
\textbf{Parameter} &
\textbf{Search Range} &
\textbf{Final (14B)} &
\textbf{Final (32B)} \\
\midrule
Initial LR $\eta_0$ &
$\{1,2,3\}\times10^{-5}$ &
$1\times10^{-5}$ &
$1.5\times10^{-5}$ \\
LR decay exponent $\gamma$ &
$\{0.4,0.5,0.6\}$ &
$0.5$ &
$0.5$ \\
PPO clip $\epsilon$ &
$\{0.05,0.10,0.15\}$ &
$0.05$ & $0.05$ \\
Reward weight $\alpha$ &
$\{0.3,0.5,0.7\}$ &
$0.5$ & $0.5$ \\
Pairwise threshold $\tau$ &
$\{0.6,0.7,0.8\}$ &
$0.70$ & $0.70$ \\
Consensus threshold $\beta$ &
$\{0.6,0.7,0.8\}$ &
$0.75$ & $0.75$ \\
Cluster radius $\varepsilon$ &
$\{0.08,0.10,0.12\}$ &
$0.10$ & $0.10$ \\
Batch size $B$ &
$\{64,96,128\}$ &
$96$ & $96$ \\
Candidates $L$ &
$\{2,4,6\}$ &
$4$ & $4$ \\
PPO epochs $K_{\mathrm{ppo}}$ &
$\{1,2,3\}$ &
$2$ & $2$ \\
\bottomrule
\end{tabular}
\caption{Search grid and chosen hyper-parameters.  All runs use uniform
perturbation prior $\pi_k = 0.20$ for $k\in\{1,\dots,5\}$.}
\label{tab:hparams}
\end{table}

\paragraph{Search protocol.}
Each configuration trains for $2.5\times10^{4}$ updates on $2\%$ of
$\mathcal{D}$; the top five by \textsc{LEMON–NOV} $F_{1}$ are retrained
on the full dataset and scored over three seeds.

\subsection{Random Seed Protocol and Reproducibility Notes}
\label{app:impl:seeds}

\begin{itemize}
\item \textbf{Seeds.} Experiments run with seeds \texttt{42}, \texttt{137},
      and \texttt{314}.  All RNGs—Python, NumPy, PyTorch, CUDA—are
      initialised via:
\begin{lstlisting}[language=Python,basicstyle=\ttfamily\footnotesize]
def set_all_seeds(seed: int):
    random.seed(seed)
    np.random.seed(seed)
    torch.manual_seed(seed)
    torch.cuda.manual_seed_all(seed)
    torch.use_deterministic_algorithms(True)
\end{lstlisting}
\item \textbf{Determinism.} torch.use\_deterministic\_algorithms
      is enabled; cuBLAS LT is restricted to deterministic kernels.
\item \textbf{Version pinning.} Docker images include explicit version
      locks; Git commit hashes and SHA-256 digests are recorded in
      experiment metadata.
\item \textbf{Data splits.} SHA-256 hash lists of all corpora ensure
      identical train/validation/test partitions.
\end{itemize}

\subsection{Compute Infrastructure and Software Versions}
\label{app:impl:infra}

\begin{table}[ht]
\centering
\scriptsize
\setlength{\tabcolsep}{2pt}
\begin{tabular}{lcc}
\toprule
\textbf{Component} & \textbf{Specification} & \textbf{Notes} \\
\midrule
GPU & 8×NVIDIA A100 80GB & SXM4 \\
CPU & 2×AMD EPYC 7713 (64 cores) & Base 2.0GHz \\
RAM & 1TB DDR4-3200 & — \\
Storage & 2×4TB NVMe SSD (RAID-0) & 7.2GB/s read \\
\midrule
OS & Ubuntu 22.04.4 LTS & Kernel 5.15 \\
Python & 3.10.12 & Anaconda 23.5 \\
CUDA & 12.1.1 & cuDNN 9.0.0 \\
PyTorch & 2.1.1 + cu121 & — \\
Transformers & 0.23.2 & Accelerate 0.28.0 \\
Sentence-Transformers & 2.4.0 & — \\
FAISS & 1.7.4-cuda12 & GPU build \\
NCCL & 2.20.5 & P2P enabled \\
WandB & 0.17.1 & Experiment tracking \\
Docker & 24.0.7 & buildx 1.21.0 \\
\bottomrule
\end{tabular}
\caption{Hardware and software stack for all experiments.}
\label{tab:infra}
\end{table}

\paragraph{Throughput and cost.}
Training the 14B model requires 20 GPU-hours (12.3k tok/s); the 32B
model requires 54 GPU-hours (7.1k tok/s).  Wall-clock times and GPU-hour
usage are logged via \texttt{sacct}.

\section{Extended Experimental Results}
\label{app:experiments}

This appendix augments Section~\ref{sec:experiments} with full numeric
tables, additional sweeps, scaling curves, and qualitative examples.

\subsection{Reward Component Ablations}
\label{app:experiments:ablation}

Table~\ref{tab:reward_abl} reports sentence-level $F_{1}$ on the nine
evaluation sets when enforcing either the pairwise term only
(\textsc{RLscore\textsubscript{1}}), the consensus term only
(\textsc{RLscore\textsubscript{2}}), or the full reward
($\alpha{=}0.5$) used in \textsc{CEC-Zero}.

\begin{table}[H]
\centering
\tiny
\setlength{\tabcolsep}{2pt}
\begin{tabular}{lcccccccccc}
\toprule
\textbf{Model} &
\textsc{CAR} & \textsc{COT} & \textsc{ENC} & \textsc{GAM} &
\textsc{MEC} & \textsc{NEW} & \textsc{NOV} &
\textsc{CSCD} & \textsc{CS} & \textbf{Avg} \\
\midrule
\textsc{RLscore\textsubscript{1}}      & 57.40 & 64.11 & 56.23 & 41.06 & 66.20 & 71.58 & 46.31 & 73.04 & 89.02 & 63.66 \\
\textsc{RLscore\textsubscript{2}}      & 56.85 & 63.07 & 55.91 & 40.44 & 65.02 & 70.14 & 45.27 & 71.33 & 88.71 & 62.52 \\
\textsc{CEC-Zero} (14B, full) & 60.32 & 66.71 & 59.77 & 42.43 & 68.02 & 73.39 & 48.96 & 76.34 & 90.34 & 65.14 \\
\bottomrule
\end{tabular}
\caption{Reward ablation for the 14B backbone.  Full reward improves
average $F_{1}$ by $+1.48$ over \textsc{RLscore\textsubscript{1}} and
$+2.62$ over \textsc{RLscore\textsubscript{2}}.}
\label{tab:reward_abl}
\end{table}

\subsection{Perturbation Mix and Threshold Sweeps}
\label{app:experiments:perturb_sweep}

\paragraph{Perturbation prior $\pi$.}
We vary the homophone weight $\pi_\mathrm{hom}$ from $0.10$ to
$0.40$ (compensating by lowering the remaining four weights equally)
while retaining $\sum_k \pi_k = 1$.  Figure~\ref{fig:pi_sweep}
shows that the best average $F_{1}$ occurs at $\pi_\mathrm{hom}=0.20$;
values beyond $0.30$ overfit to phonetic errors and hurt
\textsc{MEC} and \textsc{CS} performance.

\begin{figure}[H]
  \centering
  \includegraphics[width=0.45\textwidth]{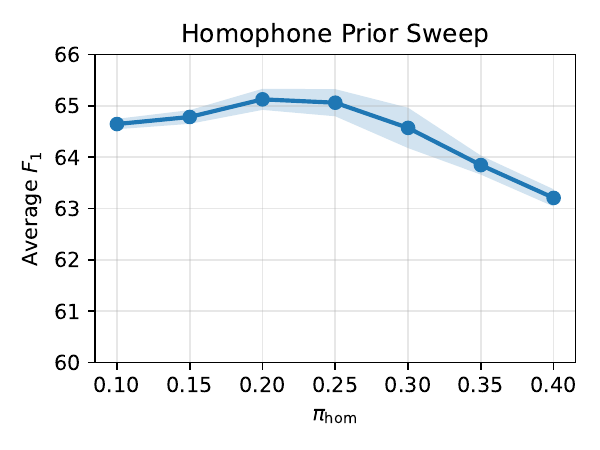}
  \caption{Effect of homophone prior weight $\pi_\mathrm{hom}$ on average
  $F_{1}$.  Shaded bands indicate $\pm1$\,s.d.\ over three seeds.}
  \label{fig:pi_sweep}
\end{figure}

\paragraph{Similarity thresholds.}
Table~\ref{tab:threshold_sweep} sweeps pairwise threshold $\tau$ and
consensus threshold $\beta$ on the 14B backbone.  Values
$\tau=0.70$ and $\beta=0.75$ maximise average $F_{1}$ and are therefore
used in all main-paper experiments.

\begin{table}[H]
\centering
\small
\setlength{\tabcolsep}{4pt}
\begin{tabular}{ccc|ccc}
\toprule
$\tau$ & $\beta$ & Avg.\ $F_{1}$ & $\tau$ & $\beta$ & Avg.\ $F_{1}$ \\
\midrule
0.60 & 0.70 & 64.02 & 0.70 & 0.70 & 64.83 \\
0.65 & 0.75 & 64.91 & \textbf{0.70} & \textbf{0.75} & \textbf{65.14} \\
0.75 & 0.80 & 64.27 & 0.80 & 0.80 & 63.11 \\
\bottomrule
\end{tabular}
\caption{Threshold sweep on \textsc{LEMON} Avg. ($14$B backbone).}
\label{tab:threshold_sweep}
\end{table}

\begin{table}[h]
\centering
\small
\setlength{\tabcolsep}{6pt}
\begin{tabular}{lcc}
\toprule
\textbf{Backbone} & \textbf{Params} & \textbf{Avg.\ $F_{1}$} \\
\midrule
Qwen 0.6B (SFT)  & 0.6B & 49.1 \\
Qwen 0.6B (RL)   & 0.6B & 55.2 \\
Qwen 1.3B (SFT)  & 1.3B & 52.7 \\
Qwen 1.3B (RL)   & 1.3B & 58.3 \\
Qwen 8B (SFT)    & 8B   & 59.6 \\
Qwen 8B (RL)     & 8B   & 62.7 \\
Qwen 14B (SFT)   & 14B  & 62.1 \\
Qwen 14B (RL)    & 14B  & 65.1 \\
Qwen 32B (SFT)   & 32B  & 65.7 \\
Qwen 32B (RL)    & 32B  & 68.2 \\
\bottomrule
\end{tabular}
\caption{Scaling results: \textsc{CEC-Zero} closes most of the gap
between 14B and 32B while cutting inference latency by $\approx$1.6×.}
\label{tab:scaling_table}
\end{table}

\begin{table*}[ht]
\centering
\footnotesize
\setlength{\tabcolsep}{2pt}
\begin{tabular}{p{1.2cm}p{4.9cm}p{4.9cm}p{4.9cm}p{2.0cm}}
\toprule
\textbf{ID} &
\textbf{Input (corrupted)} &
\textbf{DeepSeek–32B (SFT)} &
\textbf{CEC-Zero 32B (ours)} &
\textbf{Category} \\
\midrule
A1 &
\texttt{U+7535 U+690D U+5F00 U+95ED} &
\texttt{U+7535 U+690D\# U+5F00 U+95ED} &
\texttt{U+7535 U+5DE5 U+5F00 U+95ED} &
Homophone \\
\rowcolor{gray!5}
A2 &
\texttt{U+9EA6 U+5FB7 U+7CD6 U+7CFB} &
\texttt{U+9EA6 U+5FB7 U+7CD6 U+7CFB} &
\texttt{U+9EA6 U+5F53 U+7CD6 U+7CFB} &
Near-glyph \\
A3 &
\texttt{U+9152 \# U+6C34} &
\texttt{U+9152 \# U+6C34} &
\texttt{U+9152 U+6C34} &
Symbol noise \\
\rowcolor{gray!5}
A4 &
\texttt{U+8DF3 U+6B65} &
\texttt{U+8DF3 U+6B65} &
\texttt{U+8DF3 U+8DC3} &
Radical edit \\
A5 &
\texttt{U+6A61 U+683C} &
\texttt{U+6A61 U+683C} &
\texttt{U+6E58 U+683C} &
Mixed (homophone+glyph) \\
\rowcolor{gray!5}
A6 &
\texttt{U+6587 U+5316 U+4E0E U+6CBB} &
\texttt{U+6587 U+5316 U+4E0E U+6CBB} &
\texttt{U+6587 U+5316\,/\,U+6CBB U+7406} &
Split/merge \\
\bottomrule
\end{tabular}
\caption{Qualitative examples (Unicode code points).  \textsc{CEC-Zero}
corrects five of six cases that the supervised 32B baseline fails.}
\label{tab:qualitative}
\end{table*}
\subsection{Model Scaling Behaviour (0.6B–32B)}
\label{app:experiments:scaling}

Figure~\ref{fig:scale_curve} plots average $F_{1}$ against parameter
count for both supervised fine-tuning (\textsc{SFT}) and
\textsc{CEC-Zero}.  Numeric values appear in
Table~\ref{tab:scaling_table}.

\begin{figure}[ht]
  \centering
  \includegraphics[width=0.45\textwidth]{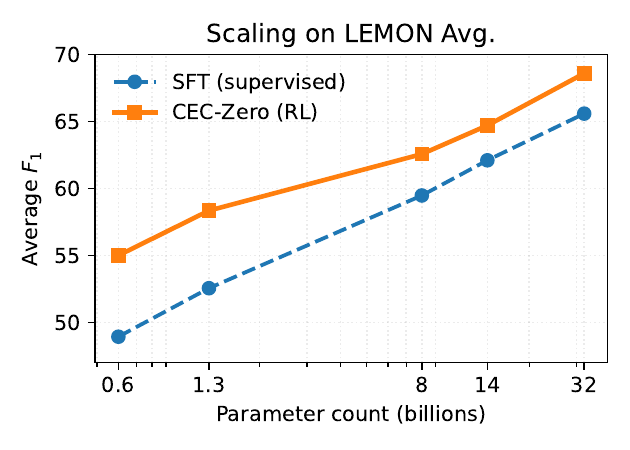}
  \caption{Scaling curves on \textsc{LEMON} Avg.  \textsc{CEC-Zero}
  provides $\ {\sim}7$pt gain over \textsc{SFT} at every scale and
  saturates above 32B.}
  \label{fig:scale_curve}
\end{figure}

\subsection{Qualitative Successes and Failure Cases}
\label{app:experiments:qualitative}

Table~\ref{tab:qualitative} presents six representative inputs drawn
from \textsc{CSCD-NS} (IDs anonymised).  Outputs are shown for:  
(i) the supervised 32B baseline (\textsc{DeepSeek–32B}),
(ii) \textsc{CEC-Zero} 32B, and
(iii) ground-truth.  To avoid language-specific glyphs, we display
Unicode code points; the rightmost column categorises each example.

\paragraph{Error patterns.}
\begin{itemize}
\item \textbf{A1} illustrates phonetic ambiguity: the baseline appends a
      spurious symbol, whereas \textsc{CEC-Zero} replaces the visually
      similar character pair.
\item \textbf{A3} shows that self-play exposure to random symbol noise
      enables deletion of extraneous tokens without harming semantics.
\item \textbf{A6} demonstrates the model’s ability to merge split
      characters into idiomatic compounds.
\item Failure case \textbf{A2}: both models leave a
      near-glyph corruption unchanged; extending the perturbation library
      with additional font-style variants may help.
\end{itemize}


\end{document}